\documentclass[onecolumn, 12pt]{IEEEtran}

\usepackage{ifpdf}

\usepackage{url}
\usepackage{cite}
\usepackage{amsmath}
\usepackage{amssymb}
\usepackage{amsthm}
\usepackage{subfig}
\usepackage{etoolbox}
\usepackage{tikz}
\usepackage{algorithm,algorithmic}
\usepackage{todonotes}
\usepackage{dsfont}\usepackage{array}\usepackage{mathrsfs}\usepackage{graphicx}

\newtheorem{lemma}{Lemma}\newtheorem{theorem}{Theorem}\newtheorem{remark}{Remark}

\newcommand\E[1]{\mathbb{E}\left[#1\right]}
\newcommand\p[1]{\mathbb{P}\left[#1\right]}
\newcommand\Var[1]{\mathrm{Var}\left[#1\right]}
\newcommand\Tr[1]{\mathrm{Tr}\left[#1\right]}
\newcommand\1[1]{\mathbb{I}_{\left\{#1\right\}}}
\newcommand\sign[1]{\mathrm{sign}(#1)}

\renewcommand{\hat}{\widehat}
\renewcommand{\tilde}{\widetilde}
\begin{document}

\title{Clustering and Inference From Pairwise Comparisons}

\author{
    \IEEEauthorblockN{Rui Wu\IEEEauthorrefmark{1}, Jiaming Xu\IEEEauthorrefmark{1}, R. Srikant\IEEEauthorrefmark{1}, Laurent Massouli\'{e}\IEEEauthorrefmark{2}, Marc Lelarge\IEEEauthorrefmark{2}, Bruce Hajek\IEEEauthorrefmark{1}}\\
    \IEEEauthorblockA{\IEEEauthorrefmark{1}University of Illinois at Urbana-Champaign\\
    }
    \IEEEauthorblockA{\IEEEauthorrefmark{2}INRIA
    }
}

\pagenumbering{arabic}
\maketitle
\begin{abstract}
Given a set of pairwise comparisons, the classical ranking problem computes a single ranking that best represents the preferences of all users. In this paper, we study the problem of inferring individual preferences, arising in the context of making personalized recommendations. In particular, we assume that there are $n$ users of $r$ types; users of the same type provide similar pairwise comparisons for $m$ items according to the Bradley-Terry model.
We propose an efficient algorithm that accurately estimates the individual preferences for almost all users, if there are $r \max \{m, n\}\log m \log^2 n$ pairwise comparisons per type, which is near optimal in sample complexity when $r$ only grows logarithmically with $m$ or $n$.
Our algorithm has three steps:
first, for each user, compute the \emph{net-win} vector
which is a projection of its $\binom{m}{2}$-dimensional vector of pairwise comparisons
onto an $m$-dimensional linear subspace;
second, cluster the users based on the net-win vectors; third, estimate a single preference for each cluster separately.
The net-win vectors are much less noisy than the high dimensional vectors of pairwise comparisons and clustering is more accurate after the projection as confirmed by numerical experiments. Moreover, we show that, when a cluster is only approximately correct, the maximum likelihood estimation for the Bradley-Terry model is still close to the true preference. 

\end{abstract}

\section{Introduction}

The question of ranking items using pairwise comparisons is of interest in many applications. Some typical examples are from sports where pairs of players play against each other and people are interested in ranking the players from past games. This type of ranking problem is usually studied using the Bradley-Terry model \cite{BradleyTerry52} where each item $i$ is associated with a score $\theta_i $ measuring its competitiveness and
\begin{align*}
  \p{\text{item $i$ is preferred over item $j$}} = \frac{e^{\theta_i}}{e^{\theta_i}+e^{\theta_j}}.
\end{align*}
For this model, the maximum likelihood estimation of the score vector $\theta$ can be solved efficiently \cite{Hunter04}.

There are other examples where comparisons are obtained implicitly. For example, when a user clicks a result from a list returned by a search engine for a given request, it implies that this user prefers this result over nearby results on the list. Similarly, when a customer buys a product from an online retailer, it implies that this customer prefers this product over previously browsed products. Businesses providing these services are interested in inferring users' rankings of items. In these examples, users can have different scores for the same item and a single score vector is insufficient to capture individual preferences. Therefore it is more appropriate to view the user preferences as generated from the mixture of Bradley-Terry models. Though this mixture model has been used in many fields (See \cite{AmmarShah14,OhShah14} and the references therein), little is known about how to cluster the users and learn the individual preferences efficiently, and how many pairwise comparions are needed for a target estimation error.

In this work, we study the following mixture Bradley-Terry model:
users are clustered into different types; users
of the same type have the same score vector; every user independently generates a few pairwise comparisons according to the Bradley-Terry model.
Notice that under our model users of the same type will have similar but not necessarily identical pairwise comparisons.
The task is to estimate the score vector for each user. Essentially, we would like to cluster the users using the observed pairwise comparisons and then estimate the score vector for each cluster. However, there are two key challenges. First, for each user, if we stack all the possible
pairwise comparisons as a vector, this comparison vector lies in a high dimensional space and only a  small number of its entries are observed.
Hence, directly clustering users based on the comparison vectors is likely to be too noisy to work well; our numerical experiments (see Section \ref{sec:experiment}) confirm as much. Second, although standard algorithms like maximum likelihood estimation \cite{Hunter04,IterativeRanking} are available for estimating the score vector once the clusters (users of the same type) are exactly found, it is still unclear how the algorithms perform when the clusters are only approximately recovered.


Our first contribution is to propose and show the effectiveness of clustering users
according to their \emph{net-win} vectors.
A net-win vector for a user is a vector of length $m$, where its $i$-th coordinate counts the number of times item $i$ is preferred over other items minus the
number of times other items  are preferred over item $i$ according to this user's pairwise comparisons. The effectiveness of net-win vectors in clustering users
relies on the following surprising fact: the means of all the comparisons vectors are close to some $(m-1)$-dimensional linear subspace; the net-win vectors
are essentially the projection of the comparisons vectors onto this low-dimensional linear subspace. We show the projection to the net-win vectors preserves
the distances between different clusters but the net-win vectors are much less noisy than the comparison vectors.  Given good separations of the net-win vectors corresponding to
different clusters, we show that a standard spectral clustering  algorithm approximately recovers the user clusters.


Our second contribution is to show that, even though the clusters have a few erroneously assigned users, the maximum likelihood estimator for the Bradley-Terry model is still close to the true score vector for this cluster. In our algorithm, as we only expect to approximately recover the user clusters, this robustness result ensures that we can still approximately recover the score vectors for most users.

The results for the clustering and estimation steps can be combined to provide a performance guarantee for the overall algorithm.
 Our algorithm  accurately estimates the score vectors for most users  with only $O\left(r\max\{m, n\}\log m \log^2 n\right)$ pairwise comparisons per cluster, where $r$ is the number of user types (clusters) and $n$ is the total number of users. When there is only one cluster, it is known that $\Omega(m)$ pairwise comparisons are required for any algorithm to accurately estimate the score vector \cite{IterativeRanking}. Also, each user needs to provide at least one pairwise comparison; otherwise there is no hope to accurately estimate the preference for this user, so at least $n$ pairwise comparisons are needed in total to learn individual preferences.
When $r$ is of order $ \log m$ or $\log n$, the sample complexity of our algorithm matches the lower bounds up to logarithmic factors.


\subsection{Related Work}


In this section, we point out some connections of our model and results to prior work. There is a vast literature on the ranking and related
 rating prediction problems; here we cover a fraction of it we see as most relevant.

Rank aggregation has been extensively studied across various disciplines including statistics, psychology, sociology, and computer science
\cite{HodgeRank11,Hirani11,MosselSort08,JordanICML2013,MSRranking}.
The Bradley-Terry (BT) model proposed in \cite{BT55,Luce59} and its various extensions are widely used for studying the rank aggregation problem \cite{Hunter04,Guiver09,Shah12,Souriani13,AzariSoufiani_icml14,Rajkumar14}.
The classical results in \cite{Hunter04} show that the likelihood function under the BT model is concave and
the ML estimator for the score vector can be efficiently found using an EM algorithm. It is further shown in \cite{IterativeRanking} and \cite{HajekOhXu14} that $\Omega(m)$ pairwise comparisons are necessary
for any algorithm to accurately infer the score vector and $O(m \log m)$ randomly chosen pairwise comparisons
is sufficient for the ML estimator. In this paper, we show the ML estimator is able to estimate the score vector
accurately even with a small number of arbitrarily corrupted pairwise comparisons.
In addition to the ML estimation, several Markov chain based iterative methods have been proposed in \cite{Dwork01, IterativeRanking},
and have been shown to accurately estimate
 the score vector with $\Omega(m \log m)$ randomly chosen pairwise comparisons, which matches the sample complexity of the ML
 estimator.

Previous works on rank aggregation, however, mostly focus on a single type of users and aim to combine the observed user preferences to output a single ranking that best represents the preferences of all users. Little is known about clustering and learning individual preferences when there are multiple types of users.  In this paper, we consider a mixed Bradley-Terry model to capture the heterogeneity of the user preferences.
 Our mixed BT model is closely related to the so-called mixed multinomial logit model studied in \cite{AmmarShah14} and \cite{OhShah14}.
 Ammar et al. \cite{AmmarShah14} studies a clustering problem similar to ours under the mixed multinomial logit model, where each user provides a set of $\ell$ favorite items instead of pairwise comparisons, and users are clustered
   based on the overlaps between the sets of $\ell$ favorite items. Under a geometric decay condition on the score vectors, the algorithm is shown to cluster users correctly with high probability if $\ell= \Omega (\log m)$ and there are only $\Omega(\text{poly}(\log m) )$ users.
 However, it is unclear whether the geometric decay condition holds in practice, and more importantly how the algorithm performs when there are a large number of users. Oh and Shah \cite{OhShah14} studies a  different problem of estimating the model parameters under the mixed multinomial logit model. A tensor decomposition based algorithm is shown to  estimate the model parameters accurately with $r^{1.5} m^3 \text{poly}(\log m)$ pairwise comparisons per component. If our algorithm is applied to estimate the model parameters, only $r m \text{poly}(\log m)$ pairwise comparisons per component are needed \footnote{Since there is no need to estimate the preferences for every user in this context, the dependency on $n$ in our sample complexity can be dropped.}.
Another mixture approach is proposed in \cite{Buhmann07} for clustering heterogeneous ranking data and an efficient EM algorithm is derived for parameter estimation. This method can take rankings of different lengths as input. However, no analytical performance guarantee is provided for the clustering.
Very recently, a nuclear norm regularization approach is proposed in \cite{Negahban14} to estimate the score vectors for all users.
By assuming each user has a unique score vector and the score matrix $\Theta^\ast$ formed by stacking all score vectors as rows is approximately of low rank $r$, they prove the estimation error  $\| \widehat{\Theta} -\Theta^\ast \|_F= o(n)$ if there are $r \max\{m,n\} \log \left( \max\{m,n\} \right)$ randomly chosen
pairwise comparisons.
However, it is not immediately clear how the nuclear norm approach performs in terms of the estimation error of the score vector for each individual.

Finally, we point out that there is a large body of work studying the related problem of rating predictions.
A popular approach is based on matrix completion
methods \cite{Candes10,Volinsky09},
the incomplete rating matrix is assumed to be of low rank.
Another line of work \cite{Dabeer12, Sigmetrics14, Massoulie11}
assumes there are multiple types of users and users of the same type provide similar ratings.
However, the rating based methods have several limitations comparing to the pairwise preference based methods.
First,  not all preferences are available in the form of ratings, while
numerical ratings can  be transformed into pairwise comparisons.
Second, ratings are user-biased, e.g.\ a user may give higher or lower ratings on average than others,
while pairwise comparisons are absolute.
Third,  pairwise comparisons are more reliable and consistent than ratings,
e.g.\ it is easier for a user to compare two items than assign scores to
them. Algorithmically, learning preferences from rankings is more challenging,
because the vectors of pairwise comparisons lie in a $\binom{m}{2}$-dimensional space,
while the vectors of ratings lie in an $m$-dimensional space.
We overcome this challenge by a simple, but non-trivial projection of the comparison vectors into a low dimensional, linear subspace.


%

\section{Problem setup}\label{sec:problem}
Consider a system with $r$ user clusters of sizes $K$ and $m$ items and let $n = rK$. Each user $u$ has a score vector for the items $\theta_u = (\theta_{u, 1}, \dots, \theta_{u, m})$, and he/she compares items according to the Bradley-Terry model:  item $i$ is preferred over item $j$ with probability $\frac{e^{\theta_{u, i}}}{e^{\theta_{u, i}}+e^{\theta_{u, j}}}$ and vice versa with probability $\frac{e^{\theta_{u, j}}}{e^{\theta_{u, i}}+e^{\theta_{u, j}}}$. Assume users in the same cluster have the same score vector and denote the common score vector for cluster $k$ by $\theta_k$. As $(\theta_{k, 1}, \dots, \theta_{k, m})$ and $(\theta_{k, 1}+C, \dots, \theta_{k, m}+C)$ for any constant $C$ define the same probability distributions of pairwise comparisons in the Bradley-Terry model, $\theta_k$ is only identifiable up to a constant shift. To eliminate the ambiguity and without loss of generality, we always shift $\theta_k$ to ensure that $\sum_i\theta_{k, i} = 0$.

The overall comparison result is represented by an $n\times\binom{m}{2}$ sample comparison matrix $R$. The $u$-th row $R_u$ is the comparison vector of users $u$. The columns are indexed by two numbers $i, j = 1, \dots, m$ with $i<j$, and the $ij$-th column corresponds to the comparisons for item $i$ and $j$. For each user $u$, and items $i$ and $j$ with $i<j$, user $u$'s comparison result of item $i$ and $j$ is sampled with probability $1-\epsilon$ independently, where $\epsilon$
is the erasure probability. Let $R_{u, ij} = 1$ if $u$ prefers $i$ over $j$, $R_{u, ij} = -1$ if $u$ prefers $j$ over $i$, and $R_{u, ij} = 0$ if $u$'s comparison is not sampled. Then
\begin{align*}
	R_{u, ij} = \begin{cases}
	1 \quad & \text{w.p. $(1-\epsilon)\frac{e^{\theta_{u, i}}}{e^{\theta_{u, i}}+e^{\theta_{u, j}}}$}\\
	0 \quad & \text{w.p. $\epsilon$}\\
	-1 \quad & \text{w.p. $(1-\epsilon)\frac{e^{\theta_{u, j}}}{e^{\theta_{u, i}}+e^{\theta_{u, j}}}$}
	\end{cases}
\end{align*}
Our goal is to estimate the score vectors $\theta_u$ from $R$.

To simplify the analysis, we will assume $\theta_k$'s are generated independently as follows: for each $k$ and $i$, generate $\theta_{k, i}^0$ i.i.d. uniformly in $[0, b]$, and then define $$\theta_{k, i} = \theta_{k, i}^0-\frac{1}{m}\sum_{i = 1}^m\theta_{k, i}.$$ Clearly, $\sum_i \theta_{k, i} = 0$ and $|\theta_{k, i}-\theta_{k, j}|\leq b$ for any $k, i$ and $j$. Notice that $\theta_{k, i}$ are not independent, and $\theta_{k, i}-\theta_{k, j} = \theta_{k, i}^0-\theta_{k, j}^0$.

\subsection{Notation and Outline}

Let $X=\sum_{t=1}^n \sigma_t u_t v_t^\top$ denote the singular value decomposition of a matrix $X \in \mathbb{R}^{n \times n}$ such that $\sigma_1 \ge \cdots \ge \sigma_n$. The spectral norm of $X$ is denoted by $\|X\|$, which is equal to the largest singular value. The best rank $r$ approximation of $X$ is defined as ${P}_r(X)=\sum_{t=1}^r \sigma_t u_t v_t^\top$.
For vectors,
let $\langle x, y \rangle$ denote the inner product between two vectors; the only norm that will be used is the usual $l_2$ norm, denoted as $\|x\|_2$.
In this paper, all vectors are row vectors.
We say an event occurs with high probability when the probability of occurrence of that event goes to one as $m$ and $n$ go to infinity.

The rest of the paper is organized as follows. Section~\ref{sec:main_results} describes our three-step
algorithm and summarizes the main results. The key idea of de-noising the comparison vectors by projection in the first step is explained in Section~\ref{sec:denoising}.
The details of the last two steps of our algorithm, i.e., user clustering and score vector estimation, are provided in Section~\ref{sec:clustering}.
All the proofs can be found in Section~\ref{sec:proof}.  The experimental results are given in Section~\ref{sec:experiment}. Section~\ref{sec:conclusion} concludes the paper. 

\section{Algorithm and Main Result}\label{sec:main_results}
Our algorithm for clustering users and inferring their preferences is presented as Algorithm~\ref{alg:algorithm_2}. The basic idea is to estimate $\theta$ in two steps:  cluster the users and then estimate a score vector for each cluster separately.


The difficulty lies in the clustering step. Recall that, in our problem, each user is represented by a comparison vector of length $\binom{m}{2}$, and
only roughly $(1-\epsilon)\binom{m}{2}$ of its entries are observed. These comparison vectors are so noisy that directly clustering them  result in poor performance, a fact which we confirm in our experiments in Section~\ref{sec:experiment}.

\begin{algorithm}
\caption{Multi-Cluster Projected Ranking}
\label{alg:algorithm_2}
\begin{algorithmic}
\STATE \textbf{Step 0: Sample splitting.} Let $\Omega$ be the support of $R$, i.e., $\Omega = \{(u, ij) | R_{u, ij}\ne 0\}$. We construct two sets $\Omega_1$ and $\Omega_2$ by independently assigning each element of $\Omega$ only to $\Omega_1$ with probability $(1+\epsilon)/4$, only to $\Omega_2$ with probability $(1+\epsilon)/4$ and to both $\Omega_1$ and $\Omega_2$ with probability $(1-\epsilon)/4$. Define $R^{(1)}_{u, ij} = R_{u, ij}\1{(u, ij)\in \Omega_1}$ and $R^{(2)}_{u, ij} = R_{u, ij}\1{(u, ij)\in \Omega_2}$.
\STATE \textbf{Step 1: Denoising.} Let $S = \frac{1}{\sqrt{m}}R^{(1)}A^\top$. The $u$-th row of $S$ is the net-win vector of user $u$.
\STATE \textbf{Step 2: User clustering.} Let $\tilde{S}$ be the rank $r$ approximation of $S$.
Construct the clusters $\hat{\mathcal{C}}_1, \dots, \hat{\mathcal{C}}_r$ sequentially. For $1\leq k\leq r$, after $\hat{\mathcal{C}}_1, \dots, \hat{\mathcal{C}}_{k-1}$ have been selected, choose an initial user $u$ not in the first $k-1$ clusters uniformly at random, and let $\hat{\mathcal{C}}_k = \{u': ||\tilde{S}_u-\tilde{S}_{u'}||_2\leq \tau\}$ where the threshold $\tau$ is specified later. Assign each remaining unclustered user to a cluster arbitrarily.
\STATE \textbf{Step 3: Score vector estimation.} Let $D_{u, ij} = \1{R^{(2)}_{u, ij} = 1}$ and $D_{u, ji} = \1{R^{(2)}_{u, ij} = -1}$ for any $u$ and $i<j$. For users in cluster $\hat{\mathcal{C}}_k$, the estimated score vector is given by $\hat{\theta}_k = \arg\max_\gamma L_k(\gamma)$, where
  $$
     L_k(\gamma) = \sum_{u\in \hat{\mathcal{C}}_k, i, j}D_{u, ij}\log\frac{e^{\gamma_i}}{e^{\gamma_i}+e^{\gamma_j}} \label{eq:defnetwin1}
  $$
\end{algorithmic}
\end{algorithm}

We overcome this difficulty by reducing the dimension of the comparison vectors. Consider user $u$ with comparison vector $R_u$. For each item $i$, define the normalized net number of wins
\begin{align*}
  S_{u, i} = &\frac{1}{\sqrt{m}}\sum_{j\ne i}\left[\1{\text{$u$ prefers $i$ over $j$}}(R)-\1{\text{$u$ prefers $j$ over $i$}}(R)\right]\\
  = & \frac{1}{\sqrt{m}}\left[-\sum_{j<i}R_{u, ji}+\sum_{j>i}R_{u, ij}\right].
\end{align*}
We call $S_u$ the net-win vector of user $u$. To simplify the notation, let $A\in \{\pm 1, 0\}^{m\times \binom{m}{2}}$ be the matrix with the $ij$-th column being $e_i-e_j$, where $e_i$ is the length $m$ vector with all $0$s except for a $1$ in the $i$-th coordinate, then it is easy to verify that
\begin{align}
  S_u = \frac{1}{\sqrt{m}}R_uA^\top. \label{eq:defnetwin1}
\end{align}

The effectiveness of net-win vectors in clustering users relies on the following surprising fact: the expected comparison vectors $\E{R_u}$ for all users,
which are $\binom{m}{2}$-dimensional nonlinear functions of the score vectors $\theta_u$, are close to the $(m-1)$-dimensional linear subspace spanned by the rows of $A$, or the row space of $A$. It suggests denoising the $R_u$'s for all users by projecting them onto the row space of $A$.
The projections of the $R_u$'s turn out to be isometric to the net-win vectors $S_u$'s.
In particular, recall our definition of $S_u$ given in \eqref{eq:defnetwin1}, the term $\frac{1}{\sqrt{m}}A^\top$ acts just like an orthogonal projection onto the row space of $A$.
We show in Section \ref{sec:denoising} that
for any two users $v,w$ in two different clusters, $\|\E{S_{v}}- \E{S_w} \|_2 \approx \| \E{R_v} -\E{R_w} \|_2$ and $\|S_u -\E{S_u}\|_2 \approx \frac{1}{\sqrt{m} } \| R_u -\E{R_u} \|_2$ for $u=v,w$. Therefore, the net-win vectors $S_u$ are much less noisy and easier to separate than the comparison vectors $R_u$.

We then cluster the net-win vectors by a standard spectral clustering algorithm in Step 2 of Algorithm~\ref{alg:algorithm_2}. Let $\{\mathcal{C}_k\}$ denote the true clusters and $\{\hat{\mathcal{C}}_k\}$ denote the clusters generated by Algorithm~\ref{alg:algorithm_2} with threshold $\tau$. Since the clusters are only identifiable up to a permutation of the indices, we define the number of errors in $\{\hat{\mathcal{C}}_k\}$ as
\begin{align*}
  \min_\pi \sum_k|\mathcal{C}_k\ \triangle\  \hat{\mathcal{C}}_{\pi(k)}|,
\end{align*}
where $\triangle$ denotes the symmetric difference of two sets.
The following theorem highlights a key contribution of the paper: the projection of comparison vectors to the row space of $A$
results in significant denoising, which allows for accurately clustering the users with only a small number of pairwise comparisons.
\begin{theorem}\label{result:clustering}
  Let $\tau = \frac{(1-\epsilon)m}{\sqrt{\log m}}$. If $b\in [b_0, 5]$ for any arbitrarily small constant $b_0>0$ or $b\geq C''m^3\log m$ for some constant $C''$, then with high probability, there exists a permutation $\pi$ such that,
  \begin{align*}
    |\mathcal{C}_k\ \triangle\ \hat{\mathcal{C}}_{\pi(k)}|\leq & \frac{512r\max\{m, n\}\log m \log n}{(1-\epsilon)m^2}, \quad \forall k
  \end{align*}
  and
  \begin{align*}
    \sum_k|\mathcal{C}_k\ \triangle\  \hat{\mathcal{C}}_{\pi(k)}|\leq & \frac{1024 r\max\{m, n\}\log m \log n}{(1-\epsilon)m^2}.
  \end{align*}
  In particular, when
  $$ (1-\epsilon)Km^2\geq 512 r\max\{m, n\}\log m \log^2 n, $$
  the fraction of misclustered users in each cluster $k$, i.e., $$
 \frac{ |\mathcal{C}_k\ \triangle\ \hat{\mathcal{C}}_{\pi(k)}|}{K} \le \frac{1}{\log n}.$$
\end{theorem}
Theorem \ref{result:clustering} implies that if $m$ and $n$ are on the
same order, roughly each user only needs to give $r^2 \text{poly}(\log m)$
pairwise comparisons to allow for the correct clustering for all but $K/\log n$ users.
Notice that a user needs to give at least one pairwise comparison.
The specific choice of $\tau$ is just for simplicity of the proof, which
can be relaxed to $\tau=C \frac{(1-\epsilon)m}{\sqrt{\log m}}$ for any constant $C>0$.
The lower bound $b \ge b_0$ is required. Consider the extreme case where $b=0$, then
the the score vectors  for all users are
all-zero vectors and the clusters are unidentifiable from pairwise comparisons.
The upper bound $b \le 5$ is an artifact of our analysis as shown by our numerical experiments.
Note that if $b=5$, then the most favorable item is preferred over the least
favorable item with probability approximately $0.993$.

After estimating the clusters, Algorithm~\ref{alg:algorithm_2} treats each cluster separately and estimates a score vector using the maximum likelihood estimation for the single cluster Bradley-Terry model. In order to avoid the dependence between Step $2$ and Step $3$, we generate two smaller samples $R^{(1)}$ and $R^{(2)}$ by subsampling $R$, and use them in the two steps respectively. It is not hard to verify that the support sets $\Omega_1$ and $\Omega_2$ are independent.

The overall performance of Algorithm~\ref{alg:algorithm_2} is characterized by the following theorem,
which shows that, when the number of pairwise comparisons is large enough, the estimations of the score vectors are accurate for most users with high probability.

\begin{theorem}\label{result:main}
Define
\begin{align*}
\eta_1=\frac{r\max\{m, n\}\log m \log n}{ (1-\epsilon)K m^2 }, \quad \eta_2= \sqrt{\frac{\log m}{(1-\epsilon)Km}}
\end{align*}
Assume $b\in[b_0, 5]$ for any arbitrarily small constant $b_0>0$, then   there exists a constant $C>0$ such that with high probability
  \begin{align*}
    \frac{||\hat{\theta}_u-\theta_u||_2}{||\theta_u||_2}\leq
     \frac{C(e^b+1)^2}{be^b}\max \left\{ \eta_1, \eta_2 \right\}
  \end{align*}
except for $512K \eta_1$ users. In particular, if $Km^2(1-\epsilon)>r\max\{m, n\}\log m \log^2 n$,
 then $\frac{||\hat{\theta}_u-\theta_u||_2 }{||\theta_u||_2}=O(\frac{1}{\log n})$ except for $O\left( \frac{K}{\log n} \right)$ users.
\end{theorem}
Theorem \ref{result:main} shows that the estimation error depends on the maximum of $\eta_1$ and $\eta_2$: $\eta_1$ characterizes the fraction of
misclustered users in a given cluster as shown by Theorem \ref{result:clustering}; $\eta_2$ characterizes the estimation error of the maximum likelihood estimation assuming the clustering is perfect.
If $r=1$, then there is no clustering error and the estimation error only depends on $\eta_2$ which matches the existing results in \cite{IterativeRanking} with a single type of user. The lower bounds in \cite{IterativeRanking} and \cite{HajekOhXu14} show that at least $\Omega(m)$ pairwise comparisons per type are needed to ensure $\frac{||\hat{\theta}_u-\theta_u||_2}{||\theta_u||_2}=o(1)$ even when clusters are known.
 Also, a user needs to provide at least one pairwise comparison for us to infer his/her preference, which means that at least $\Omega(n)$ pairwise comparisons in total are required to infer the preferences for most users.
 Theorem~\ref{result:main} shows that Algorithm~\ref{alg:algorithm_2} needs approximately $\frac{1}{2}(1-\epsilon)Km^2 = O(r\max\{m, n\}\log m\log^2 n)$ comparisons per cluster, which matches the lower bounds up to logarithmic factors if $r$ is poly-logarithmic in $n$ or $m$.

\section{Denoising using Net-win Vectors}\label{sec:denoising}
In this section, we analyze Step 1 of Algorithm~\ref{alg:algorithm_2}.
We first argue that directly clustering based on the comparison vector is too noisy to work well.
Then, we show the net-win vectors preserve the distances between different clusters from a geometric projection point of view.
Finally, we prove the net-win vectors are much less noisy than the comparison vectors.

Recall that $R_u$ is a $\binom{m}{2}$-dimensional vector of all pairwise comparisons for user $u$. For any $i<j$, the mean of the $ij$-th entry is
  \begin{align*}
    \E{R_{u, ij}}=  (1-\epsilon)\frac{e^{\theta_{u, i}}-e^{\theta_{u, j}}}{e^{\theta_{u, i}}+e^{\theta_{u, j}}}
    \triangleq  (1-\epsilon)f(\theta_{u, i}-\theta_{u, j}),
\end{align*}
where $f(x)= \frac{e^x-1}{e^x+1}$. Since two users from the same cluster
have the identical score vector, the means of their comparison vectors are also identical. With a slight abuse of notation, let $\bar{R}_k$ denote the common
means of the comparison vectors for users in cluster $k$, where $k=1, \ldots, r$. For $k \neq k'$, we call $ \| \bar{R}_k - \bar{R}_k'\|_2 $ the distance between cluster $k$ and $k'$. It is easy to check that $\| \bar{R}_k - \bar{R}_k'\|_2= \Theta\left( (1-\epsilon) m \right)$ with high probability. In other words, the distances between different clusters are roughly $\Theta\left( (1-\epsilon) m \right)$. Hence, if we observe the means of all the comparison vectors,
then clustering becomes trivial. In our problem, for each user $u$, we only observe $R_u$, which is a noisy observation of $\E{R_u}$. More specifically, since
the expected number of comparison a user provides is $(1-\epsilon)\binom{m}{2}$,
\begin{align*}
\E{\| R_u - \E{R_u} \|_2^2 } =\sum_{ i<j} \text{Var} \left[ R_{u,ij} \right] = \Theta\left( (1-\epsilon)m^2 \right).
\end{align*}
Therefore, we would expect the deviation $\| R_u - \E{R_u} \|_2 = \Theta\left( m \sqrt{1-\epsilon} \right)$, which is much larger than the distances
between different clusters given by $\Theta\left( (1-\epsilon) m \right)$. As a result, the comparison vectors for two users from the same cluster are likely to be
far apart, while the comparison vectors for two users from different clusters might happen to be close. Therefore, the comparison vectors are too noisy to be clustered directly.

In the following, we explain how to denoise the comparison vectors.
An interesting observation is that the mean of the comparison vector $\E{R_u}$ lies close to an $(m-1)$-dimensional linear subspace. In particular, using the definition of $A$, we get
$$\E{R_u} = (1-\epsilon)f(\theta_u A),$$
where for a vector $v \in \mathbb{R}^m$, $f(v) \triangleq (f(v_1), \ldots, f(v_m))$. Although $f$ is a non-linear function, we are able to show
the angle $\alpha$ between $\E{R_u}$ and the $(m-1)$-dimensional linear subspace spanned by the rows of $A$, or the row space of $A$, is not large.
To see it, let us first assume $b$ is small. Recall that $|\theta_{k, i}-\theta_{k, j}|\leq b$ for any $k, i$ and $j$. In this regime, we can linearize the function $f$ at $0$ and get
$$
  f(\theta_uA)\approx \frac{1}{2}\theta_u A,
$$
which means that $\E{R_u}$ is approximately on the row space of $A$ and the angle $\alpha\approx 0$.
Somewhat surprisingly, $\alpha$ is still not too large even if $b$ becomes so large
that the linear approximation does not work any more. Consider the extreme case
when $b\to \infty$, under our assumption that $\theta_{u, i}$ are uniformly distributed, we have $|\theta_{u, i}-\theta_{u, j}|\to\infty$, thus
\begin{align*}
  f(\theta_uA)\to \sign{\theta_uA}.
\end{align*}
The following lemma shows that $\alpha$ is approximately $30^{\circ}$ in this case.
\begin{lemma}\label{result:large_b_small_angle}
  For any $\theta_k\in \mathbb{R}^m$ and assume $\theta_{k, i}\ne \theta_{k, j}$ for any $i$ and $j$. Define row vector $\eta \in \{-1, +1\}^{\binom{m}{2}}$ as $\eta_{ij} = \sign{\theta_{k, i}-\theta_{k, j}}$. Then
  the angle between $\eta$ and the row space of $A$ is $\arccos{\sqrt{\frac{2}{3}}}$ in the limit as $m\rightarrow\infty$.
\end{lemma}
For the intermediate range of $b$, we do not have an analytical result on the upper bound of the angle $\alpha$.
Through extensive simulation as plotted in Figure~\ref{fig:angle}, we can see that the $\cos \alpha$ averaged over $100$ independent
simulations  decreases monotonically with $b$ and it is always upper bounded by $\arccos{\sqrt{\frac{2}{3}}}$.
\begin{figure}[ht]
\begin{center}
\scalebox{.6}{\includegraphics{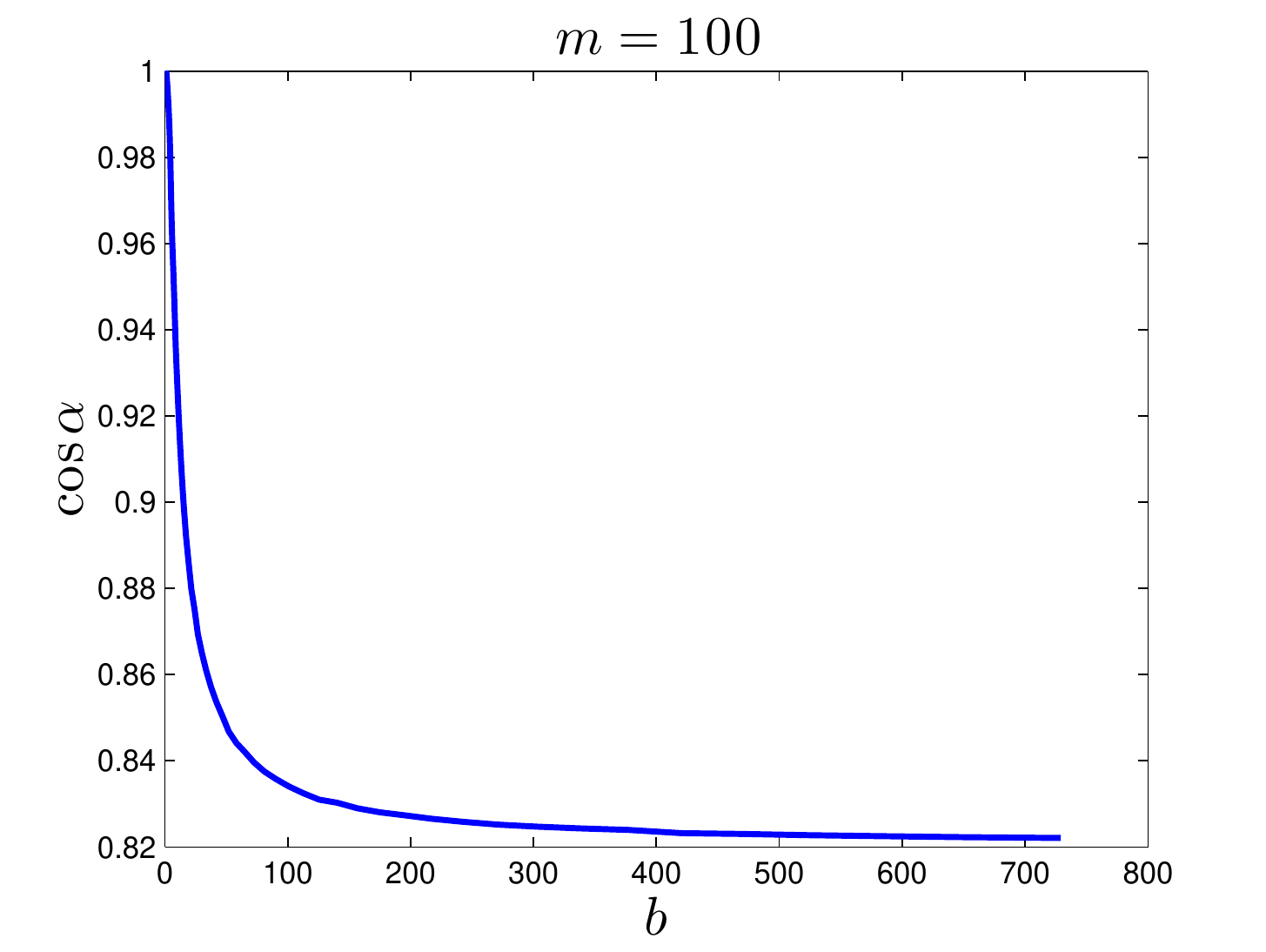}}
\end{center}
\caption{Cosine of the angle between $f(\theta_u A)$ and the row space of $A$ for various $b$.}
\label{fig:angle}
\end{figure}

The observation that $\E{R_u}$ is close to the row space of $A$ suggests that
we may denoise the comparison vectors by projecting $R_u$ onto the row space of $A$.
We show in Lemma~\ref{result:A} that the SVD of $A$ is given by $A = \sqrt{m}UV^\top$, where
$U\in \mathbb{R}^{n\times(m-1)}$ and $V\in \mathbb{R}^{\binom{m}{2}\times (m-1)}$.
Since the row vectors of $V^\top$ form an orthonormal basis of the row space of $A$,
the projection of $R_u$ onto the row space of $A$ is given by $R_uVV^\top$ and
when represented in the basis $V^\top$, the projection is simply $R_uV$.
Interestingly, we find that $R_uVV^\top$ is isometric to the
normalized net-win vectors \footnote{In Algorithm~\ref{alg:algorithm_2},
we generate two independent samples from $R$ and $S_u$ is defined using $R^{(1)}_u$. Here, we
simply write  $R^{(1)}_u$ as $R_u$ for ease of notation. } used in Algorithm~\ref{alg:algorithm_2}:
\begin{align*}
  S_u \triangleq \frac{1}{\sqrt{m}}R_u A^\top = R_uV U^\top.
\end{align*}
Since the rows of $U^\top$ form an orthonormal basis, when represented in the basis $U^\top$, $S_u$
is simply $R_u V$, which is exactly the same as $R_u VV^\top$ when represented in the basis $V^\top$.
Hence, the net-win vectors are equivalent to the projection of comparison vectors into the row space of $A$.
The benefit of using the net-win vectors instead of doing the projection is that they have a
more clear physical meaning and are easier to compute;
there is no need to compute the SVD of $A$, which is prohibitive when $m$ is large.

Since $\E{S_u} = \frac{1}{\sqrt{m}} \E{R_u} A^\top$, two users from the same cluster have the same expected net-win vectors.
With a slight abuse of notation, let $\bar{S}_k$ denote the common expected net-win vectors for users in cluster $k$, where $k=1, \ldots, r$.
The following lemma confirms that after the projection,
$\| \bar{S}_k - \bar{S}_{k'} \|_2 \asymp \| \bar{R}_k -\bar{R}_{k'} \|_2 $ and thus the projection preserves the distances
between different clusters.
\begin{lemma}\label{result:separation_of_S_bar}
	Assume $m\geq C'\log r$ for some constant $C'$. If $b\in [b_0, 5]$ for some constant $0<b_0<5$ and $b\geq C''m^3\log m$, then
	there exists some constant $C$ such that with high probability, for any $k\ne k'$,
	\begin{align*}
		||\bar{S}_k-\bar{S}_{k'}||_2 \geq C(1-\epsilon)m.
	\end{align*}
\end{lemma}
\begin{remark}
The lower bound $b \ge b_0$ is necessary. When $b$ is too small,
$\theta_k$'s all become very close to the all-zero vector and the distance between different clusters
given by $ \| \bar{R}_k -\bar{R}_{k'} \|_2$ is too small to distinguish different clusters.
Even though our theorem requires $b \in [b_0, 5]$ or $b$ is very large,
our experiment shows that the $\bar{S}_k$'s are in fact well separated for any $b\geq b_0$.
Moreover, the proof indicates that Lemma \ref{result:separation_of_S_bar}
applies to general pairwise comparison models as long as
the probability of item $i$ is preferred over item $j$ minus the probability of item $j$ is preferred over item $i$
can be parameterized as $f(\theta_i-\theta_j)$ for some sigmoid function $f$ (In BT model, $f(x)=\frac{e^x-1}{e^x+1}$);
the upper bound $ b \le 5$ changes to $ b \le c$, where $c=\min\{ b: | f(b) - f'(0) b| /b \ge f'(0)/ \sqrt{2} \}$.

\end{remark}

Next, we show the net-win vectors are much less noisy than the comparison vectors. In particular,
let $\bar{S}_u=\E{S_u}$, and then $||{S}_u-\bar{S}_u||_2 \leq 3\sqrt{\frac{1-\epsilon}{2}m\log n}$, which is much smaller than the
deviation $\| R_u - \E{R_u} \|_2 = \Theta\left( m \sqrt{1-\epsilon} \right)$.


\begin{lemma}\label{result:concentration_vector_norm}
	If $(1-\epsilon)m^2>36\log n$, then with high probability,
	$$||{S}_u-\bar{S}_u||_2\leq 3\sqrt{\frac{1-\epsilon}{2}m\log n}, \quad \forall u. $$
\end{lemma}
Notice that Lemma \ref{result:concentration_vector_norm}
is independent of the pairwise comparison model.
Together with Lemma \ref{result:separation_of_S_bar}, it shows that the
projection of comparison vectors into the row space of $A$ preserves the distances
between different clusters and at the same time dramatically reduces the noise variances.
In particular, if $m (1-\epsilon)=\Omega(\log n)$, then the net-win vectors corresponding
to different clusters are well-separated; K-means or some thresholding-based algorithm is
going to work. In the next section, we will show that spectral clustering based on the
net-win vectors does even better and works if $m(1-\epsilon)=\Omega( r^2 \text{poly}(\log n) )$
when $m$ and $n$ are on the same order.

Finally, we point out that the idea of projection or equivalently the net-win vectors introduced
in this subsection, is not specific to the BT model and is applicable to general pairwise comparison models.

\section{User clustering and score vector estimation}\label{sec:clustering}
In this section, we analyze Step 2 and 3 of Algorithm~\ref{alg:algorithm_2}.
Step 2 clusters the net-win vectors $S_u$ by a variation
of the standard spectral clustering algorithm.
After clustering the users, the algorithm estimates the score vectors for each cluster
 using sample $R^{(2)}$.
Recall that the supports of $R^{(1)}$ and $R^{(2)}$ are independent,
which is important for the analysis to decouple the two steps.

\subsection{User clustering}\label{sec:userclustering}

Step 2 of Algorithm~\ref{alg:algorithm_2} first computes the best rank $r$ approximation $\tilde{S}$ of $S$, and then clusters the rows of $\tilde{S}$ by a simple threshold based clustering algorithm. The reason we consider this threshold based clustering algorithm is that it is easy to analyze. However, in the experiments we see later, the more robust $K$-means algorithm is  used instead.

The use of  $\tilde{S}$ can  be understood from a geometric projection point of view. Let $\bar{S}=\E{S}$. Since the users from the same cluster have
the same expected score vector, the rank of $\bar{S}$ is $r$. In other words, the expected net-win vectors $\E{S_u}$ lie in an $r$-dimensional subspace of $R^m$.
Therefore, similar to the projection idea introduced in Section \ref{sec:denoising}, we may de-noise the net-win vectors $S_u$ by projecting them onto this
$r$-dimensional subspace. However, this $r$-dimensional subspace is determined by $\bar{S}$ which is unobservable. Here the key idea is that
$S$ is a perturbation of $\bar{S}$ and thus the space spanned by the top $r$ right singular vectors of $S$ is close to the desired $r$-dimensional subspace. Hence, we can de-noise the net-win vectors $S_u$  by projecting them onto the space spanned by the top $r$ right singular vectors of $S$,
which are exactly $\tilde{S}_u$. The following lemma shows that such a projection is effective in de-noising. In particular, it shows
$\| \tilde{S}- \bar{S} \|^2_F =O \left( (1-\epsilon) r \max\{m, n\} \log^3 n  \right),$
which is much smaller than the deviation bound $\| S-\bar{S} \|^2_F = O \left( (1-\epsilon) m n \log n \right)$ as shown by Lemma \ref{result:concentration_vector_norm}.

\begin{lemma}\label{result:concentration_spectral}
	If $(1-\epsilon)m^2>36\log n$, then with high probability,
\begin{align*}
||{S}-\bar{S}|| & \leq 8\sqrt{(1-\epsilon)\max\{m, n\}}\log^{3/2}n. \\
||\tilde{S}-\bar{S}||_F & \leq  16 \sqrt{2(1-\epsilon)r\max\{m, n\} }\log^{3/2}n.
\end{align*}
\end{lemma}


Using a counting argument together with Lemma \ref{result:concentration_spectral}, we can show, for most users $u$,
$\tilde{S}_u$ is close to its expected comparison vector $\bar{S}_u$.
\begin{lemma}\label{result:upper_bound_bad_user}
  Let $\tau = \frac{(1-\epsilon)m}{\sqrt{\log m}}$, then with high probability, there are at most $$\frac{512r\max\{m, n\}\log m \log n}{(1-\epsilon)m^2}$$ users such that $||\tilde{S}_u-\bar{S}_u||_2>\frac{\tau}{2}$.
\end{lemma}


Combined with the fact that the $\bar{S}_k$'s are well separated as shown in Lemma~\ref{result:separation_of_S_bar}, we get Theorem~\ref{result:clustering}.

\subsection{Score vector estimation}\label{sec:score_estimation}

In Step 3, Algorithm~\ref{alg:algorithm_2}  estimates the score vectors for each cluster separately
When there is no clustering error, the problem reduces to the inference problem for the classical Bradley-Terry model. In particular, if we let $W_{ij}$ be the number of times item $i$ is preferred over item $j$, then the ranking problem can be solved by the maximum likelihood estimation
\begin{align*}
  \hat{\theta} = \arg\max_\gamma \sum_{ij}W_{ij}\log \frac{e^{\gamma_i}}{e^{\gamma_i}+e^{\gamma_j}}.
\end{align*}
The above optimization is convex and can be solved efficiently \cite{Hunter04}. Further, the recent work \cite{IterativeRanking} provides an error bound for $\hat{\theta}$ when the pairs of items are chosen uniformly and independently.

In general, the clustering step is not perfect, but if there are sufficiently many pairwise comparisons,
Theorem \ref{result:clustering} shows that the clusters can be approximately recovered with high probability.
In this case, Algorithm~\ref{alg:algorithm_2} simply views the users in each cluster as from the same true cluster,
and again solves the optimization problem corresponding to the maximum likelihood estimation for the Bradley-Terry model for each cluster.

Take one such cluster $\hat{\mathcal{C}}$ as an example.
Recall that  $|\mathcal{C}\Delta \hat{\mathcal{C}}|$ denote the set difference between
the true cluster $\mathcal{C}$ and the estimated cluster $\hat{\mathcal{C}}$. It follows that at most
$|\mathcal{C}\Delta \hat{\mathcal{C}}|$ users in
$\hat{\mathcal{C}}$ are from other clusters and at most
$|\mathcal{C}\Delta \hat{\mathcal{C}}| $ users in $\mathcal{C}$ are assigned to wrong clusters.
To simplify the notation, we omit the subscript and use $\theta$ to denote the true score vector for the cluster $\mathcal{C}$ throughout this section.
Let $\hat{\theta}$ be the estimated score vector for cluster $\hat{\mathcal{C}}$.
The following theorem shows that when the number of comparisons is large enough,
the relative error $\frac{||\hat{\theta}-\theta||_2}{||\theta||_2}$ goes to zero when $|\mathcal{C}\Delta \hat{\mathcal{C}}|/K \to 0$.
We should emphasize that $\hat{\theta}$ is only a good approximation for the score vectors of the users from cluster $\mathcal{C}$.
\begin{theorem}\label{result:score_vector_estimation}
 Let $\hat{\mathcal{C}}$ denote an estimator of a fixed cluster $\mathcal{C}$.
  Then there exists some constant $C$ such that with high probability
  \begin{align*}
    \frac{||\hat{\theta}-\theta||_2}{||\theta||_2}\leq \frac{C(e^b+1)^2}{be^b}\max \left\{\sqrt{\frac{\log m}{(1-\epsilon)Km}}, \frac{|\mathcal{C}\Delta \hat{\mathcal{C}}|}{K} \right\}.
  \end{align*}
\end{theorem} 
Theorem \ref{result:score_vector_estimation} extends the previous results in \cite{IterativeRanking} to the setting with clustering errors. 
Notice that the error bound scales exponentially with $b$. This is likely to be an artifact of our analysis and also appears in previous results in \cite{IterativeRanking}.

\section{Proofs}\label{sec:proof}
In this section, we present the proofs for the main theorems first and then the lemmas.
The proof of Theorem~\ref{result:clustering} uses Lemma \ref{result:separation_of_S_bar} and Lemma~\ref{result:upper_bound_bad_user}.
We prove Theorem~\ref{result:main} by combining Theorem~\ref{result:clustering} and Theorem~\ref{result:score_vector_estimation}.

We first introduce some additional notation used in the proofs. Let $I$ denote the identity matrix. Let $\mathbf{1}$ denote the vector with all-one entries and $\mathbf{1} \mathbf{1}^\top$ denote the
matrix with all-one entries.
For a $m \times n$  matrix $X$ and a $m \times 1$ vector $v$, let $[X, v]$ denote the $m \times (n+1)$ matrix formed by adding $v$ as a column to the end of $X$. For two $n \times n$ matrices $X, Y$, we write $X \le Y$ if $Y-X$ is positive semi-definite.

\subsection{Proof of Theorem~\ref{result:clustering}}
  Recall that $\tau=\frac{(1-\epsilon)m}{\sqrt{\log m}}$.
   We say a user is a good user if $||\tilde{S}_u-\bar{S}_u||_2\leq \frac{\tau}{2}$.
   Under the assumption of Theorem~\ref{result:clustering}, the condition of Lemma \ref{result:separation_of_S_bar} holds.
   Then in view of  Lemma \ref{result:separation_of_S_bar}, for all good users $u$,
  \begin{align*}
    ||{S}_u-\bar{S}_u||_2\leq \frac{\tau}{2}<\frac{1}{4}||\bar{S}_k-\bar{S}_{k'}||_2.
  \end{align*}
  Let $\mathcal{I}$ be the set of good users and Lemma~\ref{result:upper_bound_bad_user} shows that the number of bad users $$|\mathcal{I}^c|\leq \frac{512r\max\{m, n\}\log m \log n}{(1-\epsilon)m^2}.$$ Following the proof of Proposition $1$ in \cite{Sigmetrics14}, we can conclude that there exists a permutation $\pi$ such that,
   $ |\mathcal{C}_k\ \triangle\ \hat{\mathcal{C}}_{\pi(k)}|\leq  |\mathcal{I}^c| $ for all $k$
  and
    $\sum_k|\mathcal{C}_k\ \triangle\  \hat{\mathcal{C}}_{\pi(k)}|\leq  2|\mathcal{I}^c|.$

 \subsection{Proof of Theorem~\ref{result:main}}
From Theorem~\ref{result:clustering}, we get that there exists a permutation $\pi$ such that,
\begin{align*}
    |\mathcal{C}_k\ \triangle\ \hat{\mathcal{C}}_{\pi(k)}|\leq & \frac{512r\max\{m, n\}\log m \log n}{(1-\epsilon)m^2}, \quad \forall k.
  \end{align*}
We then apply Theorem~\ref{result:score_vector_estimation}, and get the result of Theorem~\ref{result:main}.
If we want to achieve $\frac{||\hat{\theta}_u-\theta_u||_2 }{||\theta_u||_2}=O(\frac{1}{\log n})$ for the good users, we need
\begin{align*}
  \frac{r\max\{m, n\}\log m \log n}{(1-\epsilon)K m^2}\leq &\frac{1}{\log n}\\
  \sqrt{\frac{\log m}{(1-\epsilon)Km}} \leq &\frac{1}{\log n},
\end{align*}
which requires $Km^2(1-\epsilon)>r\max\{m, n\}\log m \log^2 n$ and $ Km (1-\epsilon)> m \log m \log^2n$, respectively.
Notice that the former condition is more stringent than the latter one; thus the clustering step needs more pairwise comparisons
than the score vector estimation step to achieve the same error rate.


\subsection{Proof of Theorem~\ref{result:score_vector_estimation}}
Let $p(m,n) \triangleq \frac{|\mathcal{C}\Delta \hat{\mathcal{C}}|}{K}$ denote the set difference between the true cluster and the estimated cluster.
  Recall that $D_{u, ij} = \1{R^{(2)}_{u, ij} = 1}$ and $D_{u, ji} = \1{R^{(2)}_{u, ij} = -1}$ are the random variable indicating $u$'s comparison result of $i<j$. The estimated score vector is given by $\hat{\theta} = \arg\max_\gamma L(\gamma)$, where
  $$
     L(\gamma) = \sum_{u, i, j}D_{u, ij}\log\frac{e^{\gamma_i}}{e^{\gamma_i}+e^{\gamma_j}}.
  $$
  Let $B_{u, ij} = B_{u, ji} = \1{R^{(2)}_{u, ij}\ne 0}$ be the random variables indicating if $u$ compared $i$ and $j$. By definition,
    \begin{align*}
        \frac{\partial L}{\partial \gamma_i} = &\sum_{u, j}(D_{u, ij}-B_{u, ij}\frac{e^{\gamma_i}}{e^{\gamma_i}+e^{\gamma_j}})\\
        \frac{\partial^2 L}{\partial \gamma_i^2} = &-\sum_j B_{ij}\frac{e^{\gamma_i}e^{\gamma_j}}{(e^{\gamma_i}+e^{\gamma_j})^2}\\
        \frac{\partial^2 L}{\partial \gamma_i\partial \gamma_j} = & B_{ij}\frac{e^{\gamma_i}e^{\gamma_j}}{(e^{\gamma_i}+e^{\gamma_j})^2},
    \end{align*}
    where $B_{ij} = \sum_u B_{u, ij}$. Let $\Delta = \hat\theta-\theta$. As $\hat{\theta}$ is the optimal solution,
    \begin{align*}
      0\leq &L(\hat\theta)-L(\theta)\\
      = & \langle \nabla L(\theta), \Delta \rangle +\frac{1}{2}\Delta^\top (\nabla^2 L(\gamma))\Delta,
    \end{align*}
    where the second step is by Taylor expansion and $\gamma = \theta+\lambda\Delta$ for some $\lambda\in [0, 1]$.
    Define $L_B=\text{diag}(B \mathbf{1} ) - B$, where $\text{diag}(v)$ denotes the diagonal matrix formed by vector $v$ and  $L_B$ is
    known as Laplacian. By Cauchy-Schwartz inequality,
    \begin{align*}
      ||\nabla L(\theta)||_2||\Delta||_2 \geq & \frac{1}{2}\Delta^\top(-\nabla^2 L(\gamma)\Delta) \\
      = & \frac{1}{4} \sum_{i,j} \left( \Delta_i - \Delta_j \right)^2 \frac{e^{\gamma_i}e^{\gamma_j} } {(e^{\gamma_i}+e^{\gamma_j})^2} \\
      \geq &\frac{e^b}{2(e^b+1)^2}\Delta^\top L_B\Delta,
    \end{align*}
    where the second inequality follows because $\frac{e^{\gamma_i}e^{\gamma_j}}{(e^{\gamma_i}+e^{\gamma_j})^2} \ge \frac{e^b}{(e^b+1)^2}$ since $|\gamma_i-\gamma_j|\leq b$ for any $i, j$.

    Let $Z_{u, ij} = D_{u, ij}-B_{u, ij}\frac{e^{\gamma_i}}{e^{\gamma_i}+e^{\gamma_j}}$. First we bound $||\nabla L(\theta)||_2$. For each $i$,
    \begin{align*}
      \frac{\partial L}{\partial \theta_i}
      = & \sum_{j, u\in C}Z_{u, ij}-\sum_{j, u\in C\setminus \hat{C}}Z_{u, ij} +\sum_{j, u\in \hat{C}\setminus C}Z_{u, ij}.
    \end{align*}
     The first term is independent of $\hat{\mathcal{C}}$. For $u\in \mathcal{C}$, $\E{Z_{u, ij}} = 0$ and $\Var{Z_{u, ij}}\leq \frac{1-\epsilon}{2}$. By Bernstein's inequality, with high probability for large $m$,
    \begin{align*}
      \bigg|\sum_{j, u\in C}Z_{u, ij} \bigg|\leq & C_1\sqrt{(1-\epsilon)Km\log m}.
    \end{align*}
    We bound the next two terms by $$\bigg|-\sum_{j, u\in C\setminus \hat{C}}Z_{u, ij} +\sum_{j, u\in \hat{C}\setminus C}Z_{u, ij}\bigg|\leq \sum_{j, u\in I^C}B_{u, ij}.$$
    Since the matrix $B$ only depends on $\Omega_2$ but not the comparison results, the right hand side above is independent of $R^{(1)}$ or $\hat{\mathcal{C}}$. As $B_{u, ij}$ are independent Bernoulli random variables with parameter $1-\epsilon$, with high probability for large $m$, $\sum_{j, u\in I^C}B_{u, ij} \leq C_2(1-\epsilon)mKp(m, n)$. Thus,
    \begin{align*}
      \bigg|-\sum_{j, u\in C\setminus \hat{C}}Z_{u, ij} +\sum_{j, u\in \hat{C}\setminus C}Z_{u, ij}\bigg|
      \leq C_2(1-\epsilon)mKp(m, n).
    \end{align*}
    Therefore,
    $$||\nabla L(\theta)||_2\leq C_2(1-\epsilon)Km^{3/2}\max\left\{\sqrt{\frac{\log m}{(1-\epsilon)Km}}, p(m, n)\right\}. $$

    Next we bound $\Delta^\top L_B\Delta$. Again by the fact that $B$ is independent of $R^{(1)}$, we can simply follow the proof of Theorem~4 in \cite{IterativeRanking} and get
    \begin{align*}
      \Delta^\top L_B\Delta\geq \frac{1}{4}(1-\epsilon)Km||\Delta||_2^2,
    \end{align*}
    with high probability for large $m$. Combining the above results, we get the upper bound on $||\Delta||_2$
    \begin{align*}
      ||\Delta||_2\leq \frac{C_3(e^b+1)^2}{e^b}\sqrt{m}\max\left\{\sqrt{\frac{\log m}{(1-\epsilon)Km}}, p(m, n)\right\}.
  \end{align*}

  On the other hand, similar to the proof of Lemma~\ref{result:separation_of_S_bar_small_b}, we can show that $||\theta||_2\geq \frac{\sqrt{mb^2}}{4}$. Therefore,
  \begin{align*}
    \frac{||\hat{\theta}-\theta||_2}{||\theta||_2} = \frac{||\Delta||_2}{||\theta||_2}\leq \frac{C(e^b+1)^2}{be^b}\max\left\{\sqrt{\frac{\log m}{(1-\epsilon)Km}}, p(m, n)\right\}.
  \end{align*}

\subsection{Proof of Lemma~\ref{result:large_b_small_angle}}

We first present a lemma on the properties of $A$ (proved in the Appendix). Recall that $A\in \{\pm 1, 0\}^{m\times \binom{m}{2}}$ is the matrix with the $ij$-th column being $e_i-e_j$, where $e_i \in \{0,1\}^m$ is a vector with all $0$s except for a $1$ in the $i$-th coordinate.
\begin{lemma}\label{result:A}
  The matrix $A$ is of rank $m-1$ with SVD $A = \sqrt{m}UV^\top$, where $U\in \mathbb{R}^{n\times(m-1)}$ and $V\in \mathbb{R}^{\binom{m}{2}\times (m-1)}$. Moreover, the $l_2$-norms of the rows of $U$ and $V$ are $\sqrt{(m-1)/m}$ and $\sqrt{2/m}$, respectively.
\end{lemma}

We are ready to prove Lemma~\ref{result:large_b_small_angle}.
Since $V^\top$ is an orthonormal basis of the row space of $A$,
the projection of a row vector $\eta$ onto this space is given by $\eta VV^\top$.
When represented in the basis $V^\top$, the projection is simply $\eta V$.
Moreover, the cosine of the angle between $\eta$ and the row space of $A$
is $\frac{||\eta V||_2}{||\eta||_2}$.

Using the properties of $A$ proved in \ref{result:A}, we have
  \begin{align*}
    ||\eta V||_2^2 = \eta VV^\top \eta^\top
     = \frac{1}{m}\eta {A^\top}A \eta^\top
     = \frac{1}{m}||A \eta^\top ||^2.
  \end{align*}
  Note that
  \begin{align*}
    (A \eta^\top)_i = \#\{j: \theta_{k, j}<\theta_{k, i}\}-\#\{j: \theta_{k, j}>\theta_{k, i}\}.
  \end{align*}
  By the assumption that $\theta_{k,i}\ne \theta_{k, j}$ for any $i$ and $j$, the vector $A\eta^\top$ is always a permutation of the deterministic vector
  \begin{align*}
    [-(m-1), -(m-3), \dots, m-3, m-1]^\top
  \end{align*}
  representing the net wins of the items. Therefore,
  \begin{align*}
    ||\eta V||_2^2 = &\frac{1}{m}||[-(m-1), -(m-3), \dots, m-3, m-1] ||^2 \\ = &\frac{1}{3}(m^2-1).
  \end{align*}
Since $||\eta||_2^2 = \frac{1}{2}m(m-1)$, the angle between $\eta$ and row space of $A$ is
$\arccos{\sqrt{\frac{2}{3}}}$ in the limit as $m\rightarrow\infty$.

\subsection{Proof of Lemma~\ref{result:separation_of_S_bar}}

We prove the lemma by considering the two regimes of $b$ separately in the following two lemmas.
\begin{lemma}\label{result:separation_of_S_bar_small_b}
	Assume $m\geq C'\log r$. If $b\in [b_0, 5]$, then a.a.s. there exists some constant $C$ such that for any $k\ne k'$,
	\begin{align*}
		||\bar{S}_k-\bar{S}_{k'}|| \geq C(1-\epsilon)m.
	\end{align*}
\end{lemma}
\begin{proof}
  Due to space constraint, we only sketch the proof in this subsection. A full proof is provided in the Appendix.

  By definition,
  $
    \bar{R}^{(1)}_{u, ij} = \frac{1-\epsilon}{2}f(\theta_{u, i}-\theta_{u, j}).
  $
  The function $f(x)$ is nonlinear but it can be approximated by the linear function $x/2$ when $x$ is close to $0$.
  Since $| \theta_{ki}- \theta_{k j} | \le b$ for all $i,j$, the maximum approximation error is given by
  $
    \delta(b) \triangleq |f(b)-\frac{b}{2}|.
  $

	By definition, for any $k$,
    \begin{align*}
      \bar{S}_k = & \frac{1-\epsilon}{2}f(\theta_k^\top A)VU^\top\\
      = & \frac{1-\epsilon}{2}\frac{1}{2}\theta_k^\top \sqrt{m}UU^\top+\frac{1-\epsilon}{2}(f(\theta_k^\top A)-\frac{1}{2}\theta_k^\top A)VU^\top.
    \end{align*}
  Then, by triangle inequality,
  \begin{align*}
    &||\bar{S}_k-\bar{S}_{k'}||_2\\
    \geq& \frac{1-\epsilon}{4}\sqrt{m}||(\theta_k-\theta_{k'})U||_2-\frac{1-\epsilon}{2} \Big[||(f(\theta_k^\top A)-\frac{1}{2}\theta_k^\top A)V||_2 \\&+ ||(f(\theta_{k'}^\top A)+ \frac{1}{2}\theta_{k'}^\top A)V||_2 \Big]\\
    \geq &\frac{1-\epsilon}{2} \sqrt{m} \Big[\frac{1}{2}||\theta_k-\theta_{k'}||_2-\frac{\delta(b)}{b}(||\theta_k||_2+||\theta_{k'}||_2)\Big].
  \end{align*}
  Using Hoeffding's inequality and Bernstein's inequality, we show that, with high probability,
  \begin{align*}
    ||\theta_k-\theta_{k'}||_2\geq \sqrt{0.9mb^2/6}, \quad
    \| \theta_k\|_2 \leq \sqrt{1.1mb^2/12}.
  \end{align*}
  Therefore,
  \begin{align*}
    ||\bar{S}_k-\bar{S}_{k'}||_2 \geq& \frac{1}{2}\left(\frac{1}{2}\sqrt{\frac{0.9}{6}}-\sqrt{\frac{1.1}{3}}\frac{\delta(b)}{b} \right)b(1-\epsilon)m \\ \ge&  C(1-\epsilon)m,
  \end{align*}
  under our assumption on $b$.
\end{proof}

\begin{lemma}\label{result:separation_of_S_bar_large_b}
	Assume $m\geq C'\log r$. If $b\geq C''m^3\log m$, then a.a.s.\ there exists some constant $C$ such that for any $k\ne k'$,
	\begin{align*}
		||\bar{S}_k-\bar{S}_{k'}|| \geq C(1-\epsilon)m.
	\end{align*}
\end{lemma}
\begin{proof}
  Due to space constraint, we only sketch the proof in this subsection. A full proof is provided in the Appendix.

  The assumption that $b$ is large implies that, with high probability, $|\theta_{k, i}-\theta_{k, j}|\geq 1$ for any $k$ and $i\ne j$.

  By definition, $\bar{S}_k = \frac{1-\epsilon}{2}f(\theta_k A)VU^\top$. Define $$\eta_{k, ij} = \1{\theta_{k, i}>\theta_{k, j}}-\1{\theta_{k, i}<\theta_{k, j}}$$ to be the signed indicator variable of the order between $\theta_{k, i}$ and $\theta_{k, j}$, and $f(\theta_kA)$ is close to $\eta_k$ when $b$ is large. Then
  \begin{align*}
    ||\bar{S}_k-\bar{S}_{k'}|| = &\frac{1-\epsilon}{2}||(f(\theta_k A)-f(\theta_{k'}A))V||_2\\
    \geq & \frac{1-\epsilon}{2}\Big[||(\eta_k-\eta_{k'})V||_2-||(f(\theta_kA)-\eta_k)V||_2\\&-||(f(\theta_{k'}A)-\eta_{k'})V||_2\Big].
  \end{align*}
  First by a counting argument we show that $||(f(\theta_kA)-\eta_k)V||_2\leq C_1\sqrt{m}$.

  Next we show that $||(\eta_k-\eta_{k'})V||_2\geq C_2m$. Observe that
  \begin{align*}
    ||(\eta_k-\eta_{k'})V||_2^2 = & ||\eta_kV||_2^2+||\eta_{k'}V||_2^2-2\eta_kVV^\top \eta_{k'}^\top \notag\\
    = & \frac{2}{3}(m^2-1)-\frac{2}{m}\eta_kA^\top A \eta_{k'}^\top, 
  \end{align*}
  where the second equality follows from Lemma~\ref{result:large_b_small_angle}. By definition of $A$, $(\eta_kA^\top)_i$ represents the number of $\theta_j$ that are smaller than $\theta_i$ minus the number of $\theta_j$ that are larger than $\theta_i$. Therefore, $\eta_kA^\top$ and $\eta_{k'}A^\top$ are independent random permutations of the deterministic vector $[-(m-1), -(m-3), \dots, m-3, m-1]$. Without loss of generality, assume $\eta_kA^\top = [-(m-1), -(m-3), \dots, m-3, m-1]$ and denote $\eta_{k'}A^\top$ by $x$ which is a random permutation of $\eta_kA^\top$. Let $Z = \eta_kA^\top A \eta_{k'}^\top = -(m-1)x_1+\dots+(m-1)x_m$ and define the martingale $y_i = \E{Z|x_1, \dots, x_i}$. Using Azuma's inequality on $Z$, we can show that $|\eta_kA^\top A \eta_{k'}^\top|\leq C_3m^{5/2}\log^{1/2}m$ with high probability, thus $||(\eta_k-\eta_{k'})V||_2\geq C_2m$.

  Combining the above two steps, we conclude that
   $ ||\bar{S}_k-\bar{S}_{k'}|| \geq C(1-\epsilon)m.$
\end{proof}

\subsection{Proof of Lemma \ref{result:concentration_vector_norm}}
Recall that for $i<j$, $V_{ij}$ denote the $(i,j)$-th row of $V$.
	Rewrite $||{S}_u-\bar{S}_u||$ as

	\begin{align*}
		||{S}_u-\bar{S}_u|| = ||\sum_{i<j}({R}^{(1)}_{u, ij}-\bar{R}^{(1)}_{u, ij})V_{ij}|| \triangleq ||\sum_{i<j}Z_{ij}||,
	\end{align*}
  where $Z_{ij} = ({R}^{(1)}_{u, ij}-\bar{R}^{(1)}_{u, ij})V_{ij}$.
	Note that $||Z_{ij}||_2 \leq ||V_{ij}||_2 = \sqrt{2/m}$. Since
  $
  	\Var{{R}^{(1)}_{u, ij}}\leq \E{({R}^{(1)}_{u, ij})^2} = \frac{1-\epsilon}{2},
  $
  we have
	\begin{align*}
		\sum_{i<j}\E{||Z_{ij}||_2^2} \leq \frac{1-\epsilon}{2}\binom{m}{2}\frac{2}{m} \leq \frac{1-\epsilon}{2}m  \triangleq \sigma^2.
	\end{align*}
	Now we apply the vector Bernstein's inequality \cite[Theorem 12]{VectorBernstein}. We choose $t=3\sigma\sqrt{\log n}$ and under our assumption it satisfies $t\sqrt{2/m}\leq \sigma^2$. Then, for any $u$,
	\begin{align*}
		\p{||{S}_u-\bar{S}_u||_2> 4\sigma\sqrt{\log n}}\leq & \p{||{S}_u-\bar{S}_u||_2> \sigma+ t}\\
    \leq &\exp(-\frac{t^2}{4\sigma^2})\leq 1/n^2.
	\end{align*}
  Applying the union bound, we get the result.

\subsection{Proof of Lemma~\ref{result:concentration_spectral}}
We bound $||S-\bar{S}||$ by the matrix Bernstein's inequality \cite{Tropp12}. Let $X_u = e_u(S_u-\bar{S}_u)$, then $S-\bar{S} = \sum_u X_u$. First we bound $||X_u||$. Since
\begin{align*}
	||X_u||^2 = & ||X_uX_u^\top|| = ||S_u-\bar{S}_u||_2^2||e_ue_u^\top|| = ||S_u-\bar{S}_u||_2^2,
\end{align*}
by Lemma~\ref{result:concentration_vector_norm}, $||X_u||\leq 3\sqrt{\frac{1-\epsilon}{2}m\log n}$ with high probability. Next we bound
$$\sigma^2 \triangleq \max\{||\sum_u\E{X_uX_u^\top}||, ||\sum_u\E{X_u^\top X_u}||\}. $$
The covariance matrix for ${S}_u$ is $\Sigma_u = V^\top D_u V$, where $D_u = \mathrm{diag}([\Var{{R}_{u, ij}}]_{ij})\leq \frac{1-\epsilon}{2}I$. Then
\begin{align*}
  ||\sum_u\E{X_uX_u^\top}|| = & ||\sum_u\E{||S_u-\bar{S}_u||_2^2e_ue_u^\top}||\\
  = & \max_u \E{||S_u-\bar{S}_u||_2^2}\\
  = & \max_u \Tr{V^\top D_u V}.
\end{align*}
Since $D_u\leq \frac{1-\epsilon}{2}I$ and using the fact that $A\leq B$ implies $\Tr{V^\top (B-A)V}\geq 0$, we get $||\sum_u\E{X_uX_u^\top}||\leq \frac{1-\epsilon}{2}m$. Similarly,
\begin{align*}
  ||\sum_u\E{X_u^\top X_u}|| = & ||\sum_u\E{(S_u-\bar{S}_u)^\top (S_u-\bar{S}_u)}||\\
  = & ||U(\sum_u V^\top D_u V)U^\top||\\
  = & ||V^\top (\sum_u D_u)V||\\
  \leq & \frac{1-\epsilon}{2}n,
\end{align*}
where the last inequality follows from $D_u\leq \frac{1-\epsilon}{2}I$ and the fact that $A\leq B$ implies $||V^\top AV||\leq ||V^\top BV||.$
Therefore, $\sigma^2\leq \frac{1-\epsilon}{2}\max\{m, n\}$. Now by applying the matrix Bernstein's inequality, we get
\begin{align*}
	||S-\bar{S}||\leq &3\max\{\max_u||X_u||\log n, \sigma\sqrt{\log n}\}\\
  \leq & 8\sqrt{(1-\epsilon)\max\{m, n\}}\log^{3/2}n.
\end{align*}
with probability at least $1-2/n$.

\begin{figure*}[!t]
\begin{center}
\scalebox{.5}{\includegraphics{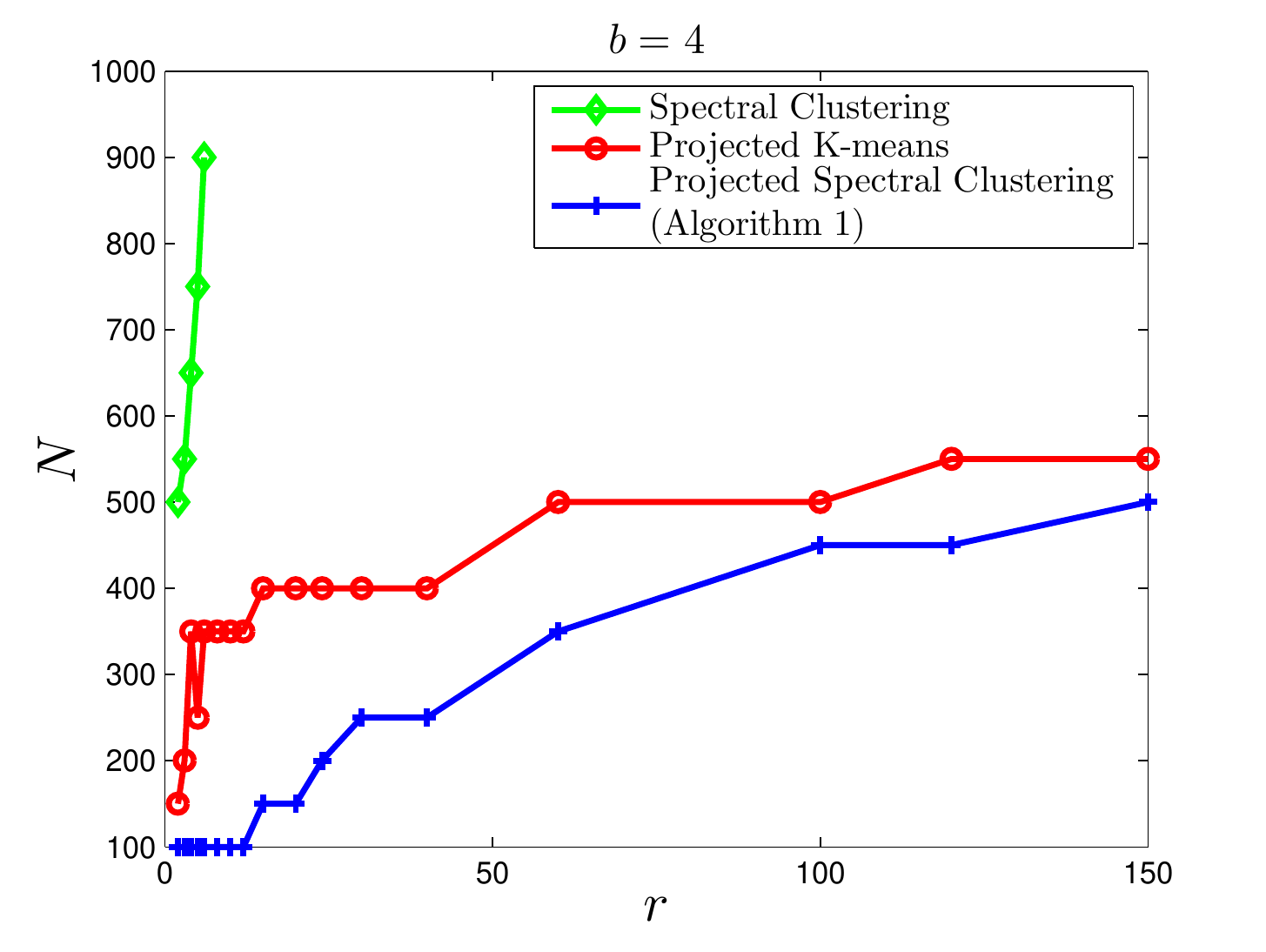}}\scalebox{.5}{\includegraphics{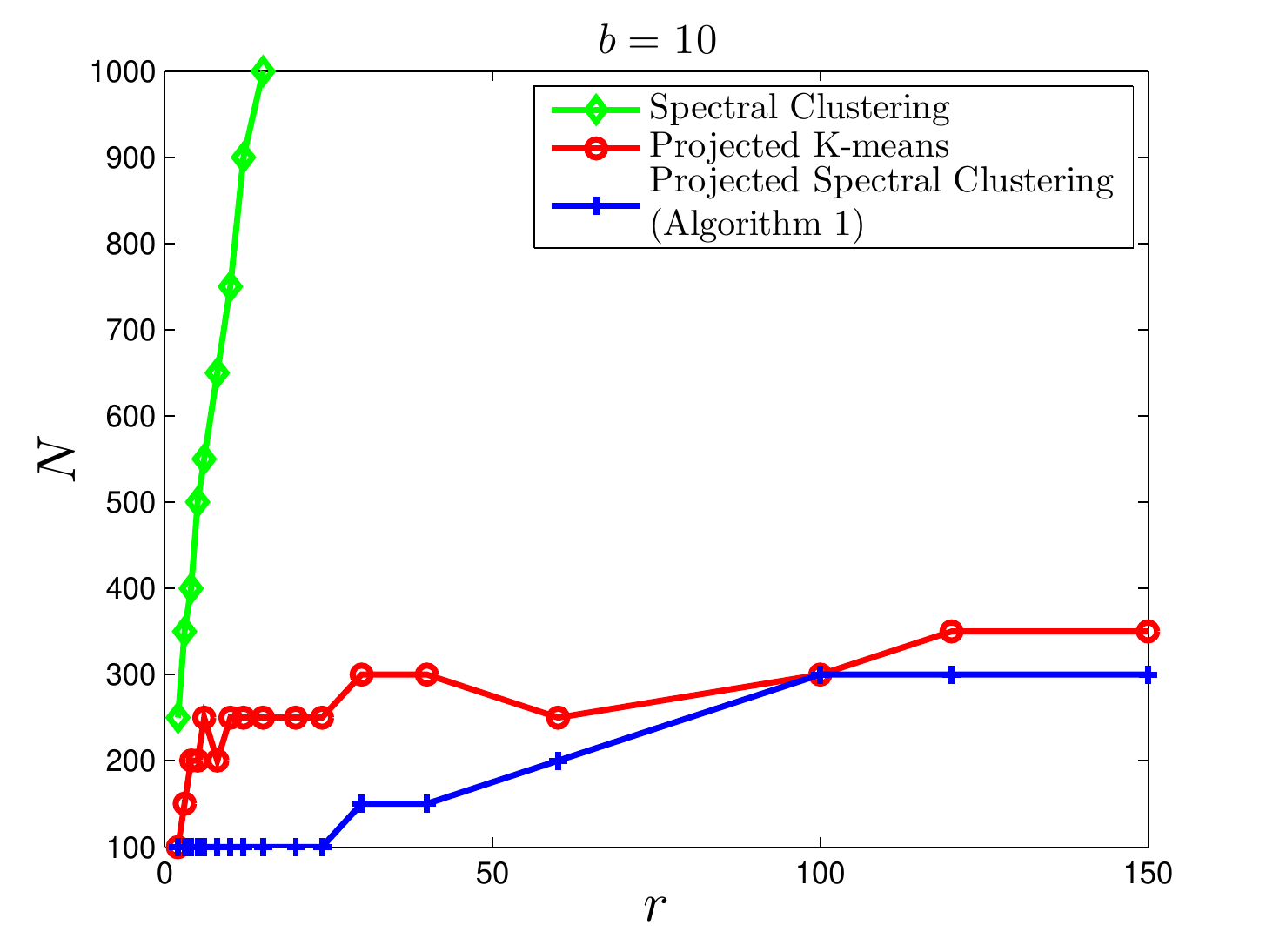}}
\end{center}
\caption{Performance Comparison of the standard spectral clustering algorithm and Algorithm~\ref{alg:algorithm_2}. The $y$-axis is expected number of comparisons $N$ provided by each user. The algorithms succeed in the parameter regime above the corresponding curves. The two algorithms using net-win vectors show significantly better clustering performance. }
\label{fig:clustering}
\end{figure*}

\subsection{Proof of Lemma~\ref{result:upper_bound_bad_user}}

We prove Lemma~\ref{result:upper_bound_bad_user} based on a counting argument.
  By Lemma~\ref{result:concentration_spectral}, with probability at least $1-2/n$, $$||{S}-\bar{S}||\leq 8\sqrt{(1-\epsilon)\max\{m, n\}}\log^{3/2} n.$$ Note that
  \begin{align*}
    ||\tilde{S}-\bar{S}|| \leq & ||\tilde{S}-{S}||+||{S}-\bar{S}||\\
    \leq & 2||{S}-\bar{S}||,
  \end{align*}
  where the second inequality follows from the definition of $\tilde{S}$ and the fact that $\bar{S}$ has rank $r$. Since the matrix $\tilde{S}-\bar{S}$ is of rank at most $2r$, we get
  \begin{align*}
    ||\tilde{S}-\bar{S}||_F^2\leq & (\sqrt{2r}||\tilde{S}-\bar{S}||)^2\\
    \leq & 8r||{S}-\bar{S}||^2\\
    \leq & 512(1-\epsilon)r\max\{m, n\}\log^3 n.
  \end{align*}
  As $||\tilde{S}-\bar{S}||_F^2 = \sum_u||\tilde{S}_u-\bar{S}_u||_2^2$, we conclude that there are at most $$\frac{512r\max\{m, n\}\log m \log n}{(1-\epsilon)m^2}$$ users with
  $
    ||\tilde{S}_u-\bar{S}_u||_2 > \frac{\tau}{2}.
  $

\section{Experiments}\label{sec:experiment}
\begin{figure*}[t!]
\begin{center}
\def\x{0.35}
\scalebox{\x}{\includegraphics{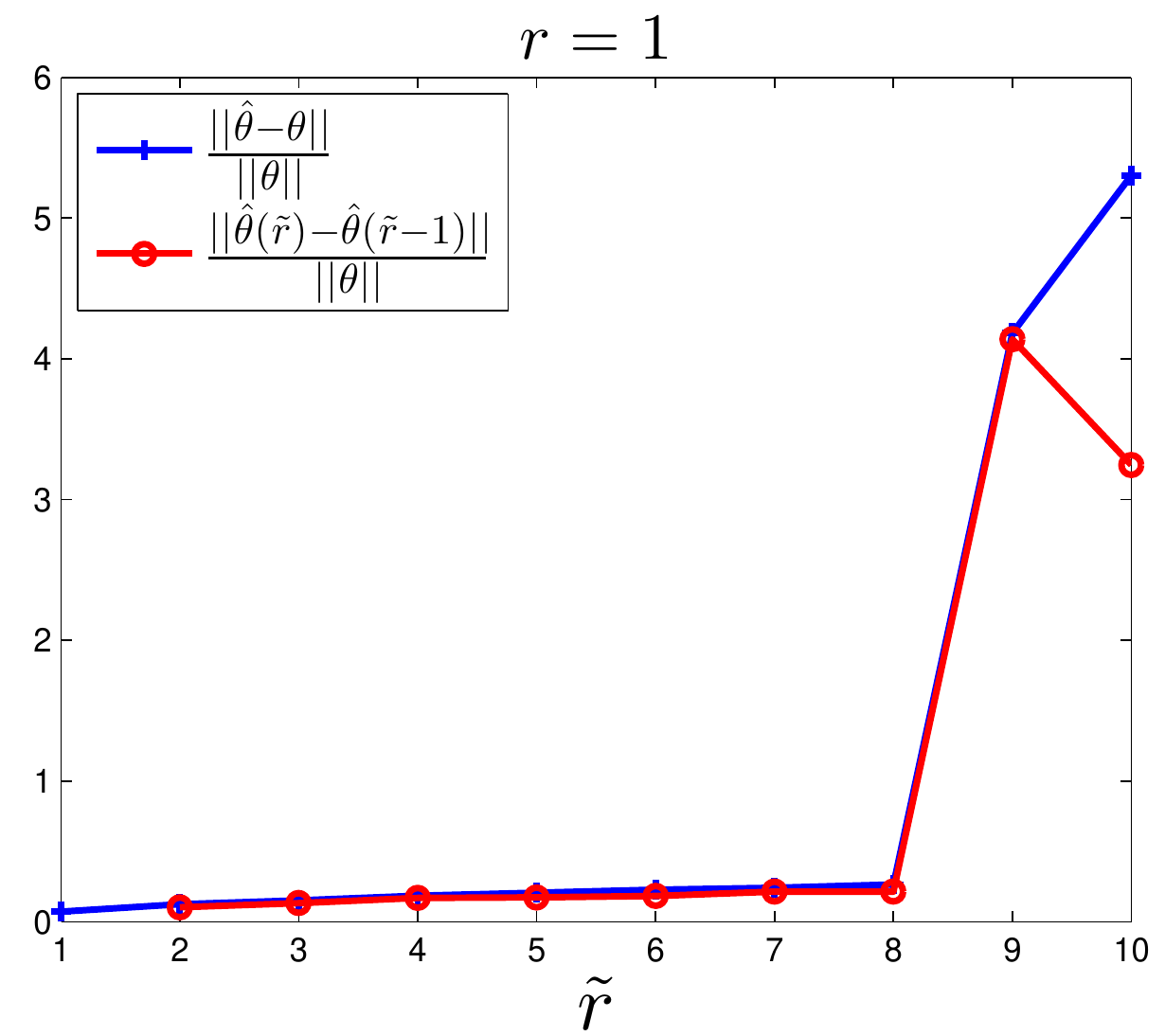}}\scalebox{\x}{\includegraphics{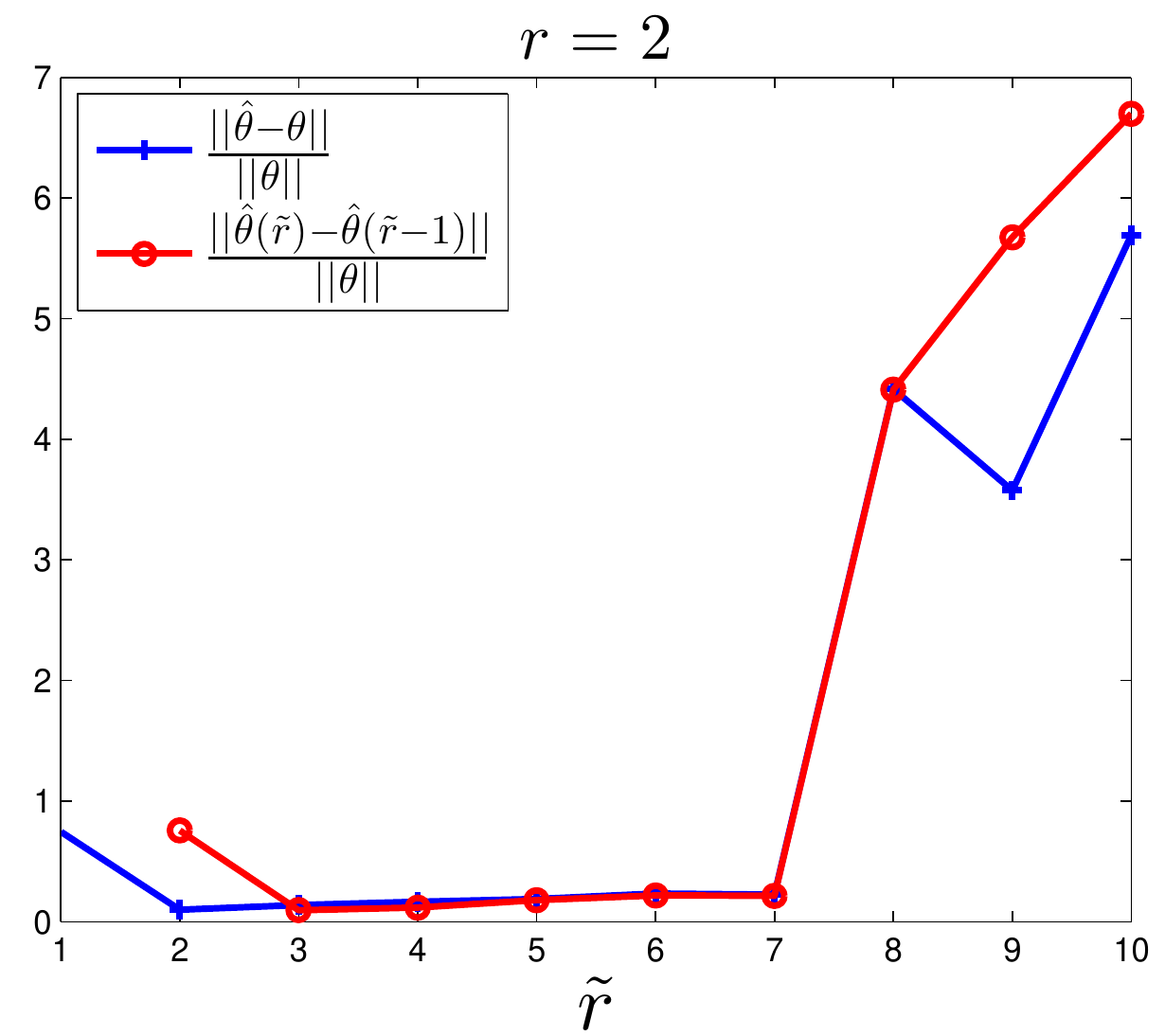}}
\scalebox{\x}{\includegraphics{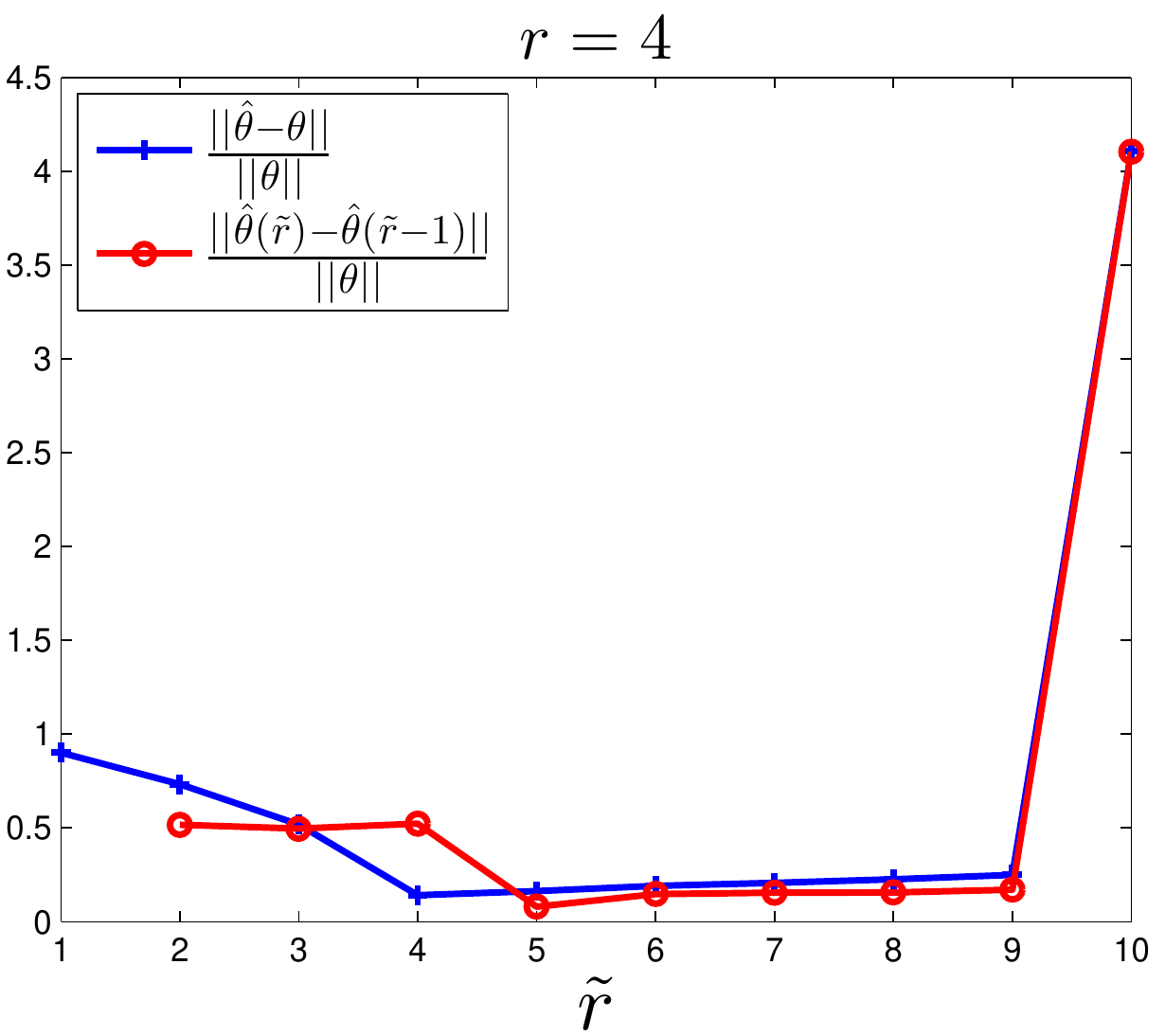}}\scalebox{\x}{\includegraphics{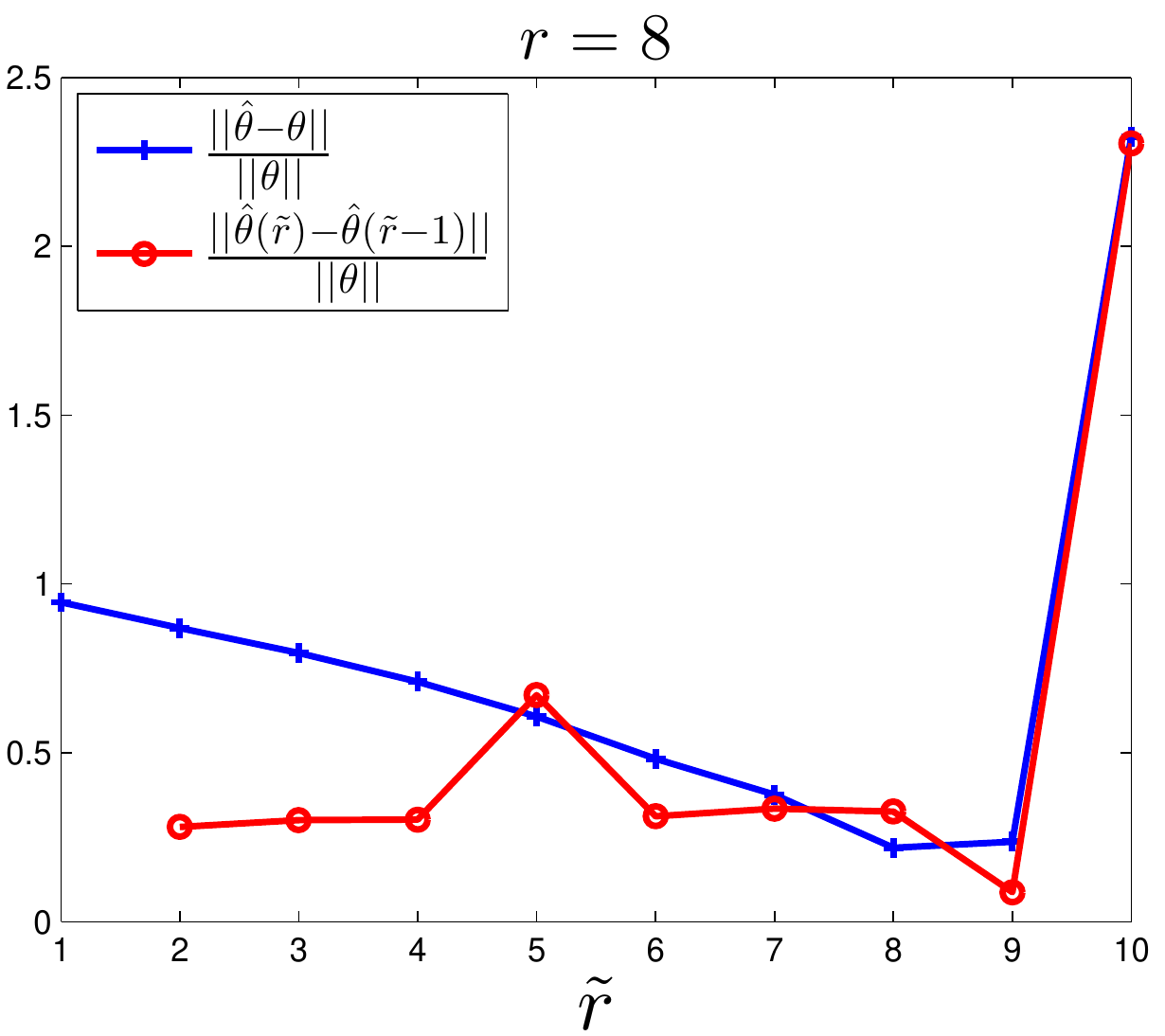}}
\end{center}
\caption{Score vector estimation for different $r$. For each $r$, the blue curve shows how the relative error $\frac{||\hat{\theta}-\theta||}{||\theta||}$ changes with $\tilde{r}$, and $\frac{||\hat{\theta}-\theta||}{||\theta||}$ is minimized when $\tilde{r} = r$. From the red curve, $r$ can be identified by looking for the $\tilde{r}$ such that the change $||\hat{\theta}(\tilde{r})-\hat{\theta}(\tilde{r}-1)||$ is minimized. }
\label{fig:score_estimation}
\end{figure*}

In this section, we illustrate the performance of our algorithm using synthetic data. Since the key idea of the paper is to show that the net-win vectors $S_u$ are easier to cluster than the comparison vectors $R_u$, in the first experiment, we compare the clustering performance of Algorithm~\ref{alg:algorithm_2} with two other clustering algorithms to verify the effectiveness of the dimension reduction. In the second experiment, we demonstrate the performance of score vector estimation and suggest a heuristic for estimating the number of clusters.

\subsection{Clustering performance comparison}

In Algorithm~\ref{alg:algorithm_2}, we cluster the rows of $\tilde{S}$ using a threshold based algorithm, which is for the ease of analysis. For the numerical experiments, we use $K$-means clustering algorithm instead, which is more robust. We initialize the centers for $K$-means clustering as follows. First, randomly pick a row as a center. Then pick the row whose minimum distance from existing centers is maximized and add it to the centers. Continue this process until we have picked $r$ centers.
In this experiment, we compare Algorithm~\ref{alg:algorithm_2} with two other clustering algorithms:
\begin{enumerate}
  \item the standard spectral clustering that applies the $K$-means algorithm to cluster the rows of $\tilde{R}$, which is the rank $r$ approximation of $R$, and
  \item the projected $K$-means that applies the $K$-means algorithm to cluster the net-win vectors $S = \frac{1}{\sqrt{m}}RA^\top$.
\end{enumerate}
We have also tried applying the $K$-means algorithm to cluster the rows of $R$. However, as the $K$-means algorithm never even approximately recover the clusters in the parameter regime considered, we do not include its performance in the plot.

We evaluate the algorithms by measuring the fraction of misclustered users. Let $\{C_k\}$ denote the true clusters and $\{\hat{C}_k\}$ denote the clusters generated by some clustering algorithm. For each $k$, we say $\hat{C}_k$ corresponds to true cluster $k'$ if the majority of users in $\hat{C}_k$ are from $C_{k'}$, and we count any user who is from a different true cluster as an error. Then the fraction of misclustered users is defined as the total number of errors divided by the total number of users.

Fix $m = n = 1200$ and let $b = 4$ or $10$. Figure~\ref{fig:clustering} shows the performance of these three algorithms. The $x$-axis is the number of clusters $r$ and $y$-axis is the average number of comparisons $N=(1-\epsilon)\binom{m}{2}$ provided by each user.
Each point on a curve shows, for the given number of clusters $r$, the smallest $N$ such that the average fraction of misclustered users of an algorithm over $50$ experiments is less than $5\%$, in which case we say the algorithm succeeds.
In other words, the algorithms succeed in the parameter regime above the corresponding curves.

Compared to the algorithms using net-win vectors, the standard spectral clustering algorithm has very poor performance. As we explained in Section \ref{sec:denoising}, the underlying reason is that the deviation $\| R_u - \E{R_u} \|_2 = \Theta\left( m \sqrt{1-\epsilon} \right)$, which is much larger than the distances between different clusters given by $\Theta\left( (1-\epsilon) m \right)$; the net-win vectors are much less noisy with $||{S}_u-\E{S_u}||_2= O\left( \sqrt{(1-\epsilon) m\log n} \right)$.
Furthermore, our Algorithm~\ref{alg:algorithm_2} performs better than the projected K-means. As we explained in Section \ref{sec:denoising}, this is due to the fact that the spectral clustering step in our algorithm further de-noises the net-win vectors by projecting them onto the space spanned by the top $r$ right singular vectors of $S$. Notice that the case $b=10$ is not covered by our theorems, but Figure~\ref{fig:clustering} shows that the clustering algorithms have similar and even better performance in this case.

%
%

\subsection{Estimating the number of clusters $r$}

In practice, the number of user clusters $r$ is usually not known \emph{a priori}. One way to get around this difficulty is to first guess the number of clusters $\tilde{r}$ and then apply our algorithm. In the experiment, we first clusters the rows of $\tilde{S}$ using the $K$-means algorithm for each $\tilde{r}$ and then apply the maximum likelihood estimation for the score vector in each cluster.

We fix $m = n = 120$, $b = 5$ and $\epsilon = 0.95$. Figure~\ref{fig:score_estimation} shows the simulation results for $r = 1, 2, 4$ and $8$.
For each $r$, the blue curve shows how the relative error $\frac{||\hat{\theta}-\theta||_2}{||\theta||_2}$ changes with $\tilde{r}$. When $\tilde{r}$ is smaller than $r$, two or more true clusters are assigned to one cluster and the error in $\hat{\theta}$ is large.
When $\tilde{r}$ is equal or slightly larger than $r$, the estimation $\hat{\theta}$ approximate $\theta$ quite well as each cluster returned by our clustering algorithm is mainly consisted of users from one true cluster. In particular, in the first plot where there is only one cluster, the relative estimation error does not grow much even for $\tilde{r} = 6$. However, when $\tilde{r}$ is too large, there will be many small clusters and the variance in $\hat{\theta}$ can be very large, which also could result large estimation error.

If we view $\hat{\theta}$ as a function of $\tilde{r}$, the red curve shows how the change of $\hat{\theta}$ in $\tilde{r}$, i.e., $||\hat{\theta}(\tilde{r})-\hat{\theta}(\tilde{r}-1)||_2$, changes with $\tilde{r}$. For comparison purpose, we normalize this difference by $||\theta||_2$. From the experiment, a good heuristic for identifying the number of clusters $r$ is by looking for the $\tilde{r}$ such that the change $||\hat{\theta}(\tilde{r})-\hat{\theta}(\tilde{r}-1)||_2$ is minimized.

%

\section{Conclusions}\label{sec:conclusion}
This paper studies the problem of clustering and ranking items with pairwise comparisons obtained from multiple types of users. 
The key idea is that projecting the comparison vectors onto a particular low dimensional linear subspace  significantly reduces the noise and improves the clustering performance; the projection can be efficiently computed by calculating the net-win vectors for each user. 
Our proofs require $b\in [b_0, 5]$ to show the projection preserves the distances between clusters, while
the experiments indicate that the means of the net-win vectors $\bar{S}_k$'s for different clusters are well separated for any $b\geq b_0$. An interesting future work is
to prove that this result indeed holds for a wide range of $b$. 
Also, under a deterministic sampling model where the set of observed entries of the comparison matrix $R$ is fixed,  
the net-win vectors are known to be the sufficient statistics for estimating the score vector
under the BT model if the clusters are known; it is interesting to prove similar results when the clusters are unknown.  


\bibliographystyle{abbrv}
\bibliography{./BT_cluster,./ranking,./BibCDRecommender}

\begin{thebibliography}{10}

\bibitem{AmmarShah14}
A.~Ammar, S.~Oh, D.~Shah, and L.~Voloch.
\newblock What's your choice? learning the mixed multi-nomial logit model.
\newblock 2014.

\bibitem{AzariSoufiani_icml14}
H.~{{Azari Soufiani}}, D.~Parkes, and L.~Xia.
\newblock Computing parametric ranking models via rank-breaking.
\newblock In {\em Proceedings of the International Conference on Machine
  Learning}, 2014.

\bibitem{Dabeer12}
K.~Barman and O.~Dabeer.
\newblock Analysis of a collaborative filter based on popularity amongst
  neighbors.
\newblock {\em IEEE Transactions on Information Theory}, 58(12):7110--7134,
  2012.

\bibitem{BradleyTerry52}
R.~A. Bradley and M.~E. Terry.
\newblock {Rank Analysis of Incomplete Block Designs: I. The Method of Paired
  Comparisons}.
\newblock {\em Biometrika}, 39(3/4), 1952.

\bibitem{BT55}
R.~A. Bradley and M.~E. Terry.
\newblock Rank analysis of incomplete block designs: I. the method of paired
  comparisons.
\newblock {\em Biometrika}, 39(3/4):324--345, 1955.

\bibitem{MosselSort08}
M.~Braverman and E.~Mossel.
\newblock Noisy sorting without resampling.
\newblock In {\em Proceedings of the Nineteenth Annual ACM-SIAM Symposium on
  Discrete Algorithms}, SODA '08, pages 268--276, 2008.

\bibitem{Buhmann07}
L.~M. Busse, P.~Orbanz, and J.~M. Buhmann.
\newblock Cluster analysis of heterogeneous rank data.
\newblock In {\em Proceedings of the 24th International Conference on Machine
  Learning}, ICML '07, pages 113--120, 2007.

\bibitem{Candes10}
E.~J. Cand\`{e}s and T.~Tao.
\newblock The power of convex relaxation: near-optimal matrix completion.
\newblock {\em IEEE Trans. Inf. Theor.}, 56(5):2053--2080, May 2010.

\bibitem{MSRranking}
X.~Chen, P.~N. Bennett, K.~Collins-Thompson, and E.~Horvitz.
\newblock Pairwise ranking aggregation in a crowdsourced setting.
\newblock In {\em Proceedings of the Sixth ACM International Conference on Web
  Search and Data Mining}, WSDM '13, pages 193--202, 2013.

\bibitem{Dwork01}
C.~Dwork, R.~Kumar, M.~Naor, and D.~Sivakumar.
\newblock Rank aggregation methods for the web.
\newblock In {\em Proceedings of the 10th International Conference on World
  Wide Web}, WWW '01, pages 613--622, 2001.

\bibitem{VectorBernstein}
D.~Gross.
\newblock Recovering low-rank matrices from few coefficients in any basis.
\newblock {\em IEEE Trans. Inf. Theor.}, 57(3):1548--1566, Mar. 2011.

\bibitem{Guiver09}
J.~Guiver and E.~Snelson.
\newblock {B}ayesian inference for {P}lackett-{L}uce ranking models.
\newblock In {\em Proceedings of the 26th Annual International Conference on
  Machine Learning}, pages 377--384, New York, NY, USA, 2009.

\bibitem{HajekOhXu14}
B.~Hajek, S.~Oh, and J.~Xu.
\newblock Minimax-optimal inference from partial rankings.
\newblock {\em Advances in Neural Information Processing Systems}, 2014.

\bibitem{Hirani11}
A.~N. Hirani, K.~Kalyanaraman, and S.~Watts.
\newblock Least squares ranking on graphs.
\newblock 2011.

\bibitem{Hunter04}
D.~R. Hunter.
\newblock {MM algorithms for generalized Bradley-Terry models}.
\newblock {\em The Annals of Statistics}, 32:384--406, 2004.

\bibitem{HodgeRank11}
X.~Jiang, L.-H. Lim, Y.~Yao, and Y.~Ye.
\newblock Statistical ranking and combinatorial hodge theory.
\newblock {\em Math. Program.}, 127(1):203--244, Mar. 2011.

\bibitem{Volinsky09}
Y.~Koren, R.~Bell, and C.~Volinsky.
\newblock Matrix factorization techniques for recommender systems.
\newblock {\em Computer}, 42(8):30--37, 2009.

\bibitem{Negahban14}
Y.~Lu and S.~N. Negahban.
\newblock Individualized rank aggregation using nuclear norm regularization.
\newblock {\em arXiv preprint arXiv:1410.0860}, 2014.

\bibitem{Luce59}
D.~R. Luce.
\newblock {\em {Individual Choice Behavior}}.
\newblock Wiley, New York, 1959.

\bibitem{Massoulie11}
L.~Massouli\'e and D.-C. Tomozei.
\newblock Distributed user profiling via spectral methods.
\newblock 2011.

\bibitem{Shah12}
S.~Negahban, S.~Oh, and D.~Shah.
\newblock Rank centrality: Ranking from pair-wise comparisons.
\newblock {\em arXiv:1209.1688}, 2012.

\bibitem{IterativeRanking}
S.~Negahban, S.~Oh, and D.~Shah.
\newblock Rank centrality: Ranking from pair-wise comparisons.
\newblock 2014.

\bibitem{OhShah14}
S.~Oh and D.~Shah.
\newblock Learning mixed multinomial logit model from ordinal data.
\newblock 2014.

\bibitem{Rajkumar14}
A.~Rajkumar and S.~Agarwal.
\newblock A statistical convergence perspective of algorithms for rank
  aggregation from pairwise data.
\newblock In {\em Proceedings of the International Conference on Machine
  Learning}, 2014.

\bibitem{Souriani13}
H.~A. Soufiani, W.~Chen, D.~C. Parkes, and L.~Xia.
\newblock Generalized method-of-moments for rank aggregation.
\newblock In {\em Advances in Neural Information Processing Systems 26}, pages
  2706--2714. 2013.

\bibitem{Tropp12}
J.~Tropp.
\newblock User-friendly tail bounds for sums of random matrices.
\newblock {\em Foundations of Computational Mathematics}, 12(4):389--434, 2012.

\bibitem{JordanICML2013}
F.~Wauthier, M.~Jordan, and N.~Jojic.
\newblock Efficient ranking from pairwise comparisons.
\newblock In {\em Proceedings of the 30th International Conference on Machine
  Learning (ICML-13)}, volume~28, pages 109--117, May 2013.

\bibitem{Sigmetrics14}
J.~Xu, R.~Wu, K.~Zhu, B.~Hajek, R.~Srikant, and L.~Ying.
\newblock Jointly clustering rows and columns of binary matrices: Algorithms
  and trade-offs.
\newblock {\em Proceedings of ACM Sigmetrics}, 2014.

\end{thebibliography}

\appendix
\section{Proof of Lemma~6}

Note that $AA^\top = m I - J \triangleq L_m$, which is the Laplacian of the complete graph.  It is easy to check that
 the eigenvalues of $L_m$ are $0, m$, and the eigenvector corresponding to the zero eigenvalue is given by $\frac{1}{\sqrt{m}} \mathbf{1}$.
 Therefore, $A$ is of rank $m-1$ and all its nonzero singular values are $\sqrt{m}$. Moreover, the SVD of $A$ is $A = \sqrt{m}UV^\top$ and
 $[U, \frac{1}{\sqrt{m}} \mathbf{1}]$ is an orthogonal matrix.  Let $U_i$ be the $i$-th row of $U$ and then $\|U_i\|_2 = \sqrt{(m-1)/m}$
 for $i=1, \ldots, m-1$. For $1 \le i < j \le m$, let $V_{ij}$ denote the $(i,j)$-th row of $V$ and $(A^\top U)_{ij}$ denote the $(i,j)$-th
 row of $A^\top U$. Since $V= A^\top U/ \sqrt{m}$, it follows that $V_{ij}=\frac{1}{\sqrt{m}} (U_i-U_j)$ and
 \begin{align*}
 ||V_{ij}||_2 &= \frac{1}{\sqrt{m}} ||U_i-U_j||_2 = \frac{1}{\sqrt{m}} \big|\big| [U_i, \frac{1}{\sqrt{m } }]-[U_j, \frac{1}{\sqrt{m}} ]\big|\big|_2 \\
 &= \sqrt{\frac{2}{m}}.
\end{align*}
\section{Proof of Lemma~7}
  By definition,
  $
    \bar{R}^{(1)}_{u, ij} = \frac{1-\epsilon}{2}f(\theta_{u, i}-\theta_{u, j}).
  $
  The function $f(x)$ is nonlinear and can be approximated by the linear function $x/2$ when $x$ is close to $0$.
  Since $| \theta_{ki}- \theta_{k j} | \le b$ for all $i,j$, the maximum approximation error is given by
  $
    \delta(b) \triangleq |f(b)-\frac{b}{2}|.
  $

	By definition, for any $k$,
    \begin{align*}
      \bar{S}_k = & \frac{1-\epsilon}{2}f(\theta_k^\top A)VU^\top\\
      = & \frac{1-\epsilon}{2}\frac{1}{2}\theta_k^\top \sqrt{m}UU^\top+\frac{1-\epsilon}{2}(f(\theta_k^\top A)-\frac{1}{2}\theta_k^\top A)VU^\top.
    \end{align*}
  Then, the difference between $\bar{S}_k$ and $\bar{S}_{k'}$ is lower bounded by
  \begin{align*}
    &||\bar{S}_k-\bar{S}_{k'}||_2\\
    \geq& \frac{1-\epsilon}{4}\sqrt{m}||(\theta_k-\theta_{k'})U||_2-\frac{1-\epsilon}{2} \Big[||(f(\theta_k^\top A)-\frac{1}{2}\theta_k^\top A)V||_2 \\&+ ||(f(\theta_{k'}^\top A)+ \frac{1}{2}\theta_{k'}^\top A)V||_2 \Big].
  \end{align*}
  As $\sum_i\theta_{k, i} = \sum_i\theta_{k', i} = 0$,
  \begin{align*}
    ||(\theta_k-\theta_{k'})U||_2 = ||(\theta_k-\theta_{k'})[U, \frac{1}{\sqrt{m}}\mathbf{1}]||_2 = ||\theta_k-\theta_{k'}||_2.
  \end{align*}
  Using the fact that
  \begin{align*}
    |f((\theta_i-\theta_j)-\frac{1}{2}(\theta_i-\theta_j)|\leq \frac{\delta(b)}{b}|\theta_i-\theta_j|,
  \end{align*}
  we get
  \begin{align*}
    ||f(\theta_k^\top A)-\frac{1}{2}\theta_k^\top A||_2\leq &\frac{\delta(b)}{b}\sqrt{\sum_{i<j}(\theta_{k, i}-\theta_{k, j})^2}\\
    = &\frac{\delta(b)}{b}\sqrt{m||\theta_k||_2^2}.
  \end{align*}
  Therefore,
  \begin{align*}
    ||\bar{S}_k-\bar{S}_{k'}||_2 \geq &(1-\epsilon) \sqrt{m} \Big[\frac{1}{2}||\theta_k-\theta_{k'}||_2\\&-\frac{\delta(b)}{b}(||\theta_k||_2+||\theta_{k'}||_2)\Big].
  \end{align*}
  First we bound $||\theta_k-\theta_{k'}||_2$. Recall that $\theta_k$ is the centered version of $\theta_k^0$, and $\theta_{k, i}^0$ are generated i.i.d. uniformly in $[0, b]$. When $m>C_1\log r$, by Hoeffding's inequality,
  $$ \bigg|\sum_i\theta_{k, i}^0-\frac{mb}{2}\bigg|\leq C_2\sqrt{m\log r}, \quad \bigg|\sum_i\theta_{k', i}^0-\frac{mb}{2}\bigg|\leq C_2\sqrt{m\log r}$$
  with high probability. By definition,
  \begin{align*}
    ||\theta_k-\theta_{k'}||_2\geq & ||\theta_{k}^0-\theta_{k'}^0||_2-||\frac{1}{m}(\sum_i\theta_{k, i}^0-\sum_i\theta_{k', i}^0)\mathbf{1}||_2\\
    \geq & ||\theta_{k}^0-\theta_{k'}^0||_2-2C_2\sqrt{\log r}.
  \end{align*}
  To bound $||\theta_{k}^0-\theta_{k'}^0||_2$, note that $$\E{||\theta_{k}^0-\theta_{k'}^0||_2^2} = \E{\sum_i(\theta_{k, i}^0-\theta_{k', i}^0)^2} = \frac{mb^2}{6}.$$ Define $X_i = (\theta_{k, i}^0-\theta_{k', i}^0)^2-\frac{b^2}{6}$. Then $|X_i|\leq b^2, \E{X_i} = 0$ and
  \begin{align*}
    \E{X_i^2} = &\E{ (\theta_{k, i}^0-\theta_{k', i}^0)^4}-\frac{b^4}{36}\\
    = &\E{ \left(\theta_{k, i}^0-\frac{b}{2} \right)^4}+6\E{ \left(\theta_{k, i}^0-\frac{b}{2} \right)^2 \left(\theta_{k', i}^0-\frac{b}{2} \right)^2} \\&+\E{ \left(\theta_{k', i}^0-\frac{b}{2} \right)^4}-\frac{b^4}{36}\\
    \leq & \frac{b^4}{80}+\frac{b^4}{6}+\frac{b^4}{80}-\frac{b^4}{36}\leq \frac{b^4}{6}.
  \end{align*}
  By Bernstein's inequality, when $m\geq C_3\log r$, with high probability,
  \begin{align*}
    \left|||\theta_{k}^0-\theta_{k'}^0||_2^2-\frac{mb^2}{6}\right|\leq C_4\sqrt{m\log r}b^2.
  \end{align*}
  When $m$ is large enough, we have $||\theta_k-\theta_{k'}||_2\geq \sqrt{0.9mb^2/6}$ with high probability.

  Next we bound $||\theta_k||_2$. By definition, $\| \theta_k\|_2 \le \| \theta_k^0 \|_2$. Note that $\E{\| \theta_k^0 \|_2^2}= mb^2/12$.
  Following the similar argument as above, we can show that, when $m\geq C_5\log r$, $\| \theta_k\|_2 \le ||\theta_k^0||_2\leq \sqrt{1.1mb^2/12}$ with high probability.

  Combining the two inequalities, we get
  \begin{align*}
    ||\bar{S}_k-\bar{S}_{k'}||_2 \geq& \frac{1}{2}\left(\frac{1}{2}\sqrt{\frac{0.9}{6}}-\sqrt{\frac{1.1}{3}}\frac{\delta(b)}{b} \right)b(1-\epsilon)m \\ \ge&  C(1-\epsilon)m,
  \end{align*}
  where the last inequality holds because $b\in [b_0, 5]$, $\delta(b)/b$ increases with $b$ and $\frac{1}{2}\sqrt{\frac{0.9}{6}} >\sqrt{\frac{1.1}{3}}\frac{\delta(5)}{5}$.

\section{Proof of Lemma~8}
  The assumption that $b$ is large implies that $|\theta_{k, i}-\theta_{k, j}|$ is large for any $k$ and $i<j$. To show this, we note that
  \begin{align*}
    \p{|\theta_{k, i}-\theta_{k, j}|<1}\leq \frac{2}{b}.
  \end{align*}
  Then by union bound we get
  \begin{align}
    \p{\forall k, i<j, |\theta_{k, i}-\theta_{k, j}|\geq 1} \geq 1-\frac{m^3}{b}\geq 1-\frac{C''}{\log m}. \label{equ:theta_sep}
  \end{align}
  In the following we will assume $\theta_{k, i}\ne \theta_{k, j}$ for $i\ne j$. By definition, $\bar{S}_k = \frac{1-\epsilon}{2}f(\theta_k A)VU^\top$. Define $\eta_{k, ij} = \1{\theta_{k, i}>\theta_{k, j}}-\1{\theta_{k, i}<\theta_{k, j}}$ to be the signed indicator variable of the order between $\theta_{k, i}$ and $\theta_{k, j}$. Then
  \begin{align*}
    ||\bar{S}_k-\bar{S}_{k'}|| = &\frac{1-\epsilon}{2}||(f(\theta_k A)-f(\theta_{k'}A))V||_2\\
    \geq & \frac{1-\epsilon}{2}\Big[||(\eta_k-\eta_{k'})V||_2-||(f(\theta_kA)-\eta_k)V||_2\\&-||(f(\theta_{k'}A)-\eta_{k'})V||_2\Big]\\
    \geq & \frac{1-\epsilon}{2}\Big[||(\eta_k-\eta_{k'})V||_2-||f(\theta_kA)-\eta_k||_2 \\ &-||f(\theta_{k'}A)-\eta_{k'}||_2\Big].
  \end{align*}
  First we show that $||f(\theta_kA)-\eta_k||_2\leq C_1\sqrt{m}$. When $|\theta_{k, i}-\theta_{k, j}|\geq t$,
  \begin{align*}
    |f(\theta_{k, i}-\theta_{k, j})-\eta_{k, ij}|\leq \frac{2}{e^t+1}\leq 2e^{-t}.
  \end{align*}
  According to \eqref{equ:theta_sep}, for any integer $1\leq t\leq m$, there are $m-t$ pairs of $\theta_{k, i}$ and $\theta_{k, i}$ separated at least by $t$. Therefore,
  \begin{align*}
    ||f(\theta_kA)-\eta_k||_2^2\leq \sum_{t = 1}^m (m-t)4e^{-2t}\leq C_1m.
  \end{align*}
  We bound $||f(\theta_{k'}A)-\eta_{k'}||_2$ similarly.

  Next we show that $||(\eta_k-\eta_{k'})V||_2\geq C_2m$. Observe that
  \begin{align}
    ||(\eta_k-\eta_{k'})V||_2^2 = & ||\eta_kV||_2^2+||\eta_{k'}V||_2^2-2\eta_kVV^\top \eta_{k'}^\top \notag\\
    = & \frac{2}{3}(m^2-1)-\frac{2}{m}\eta_kA^\top A \eta_{k'}^\top, \label{equ:eta_seq}
  \end{align}
  where the second equality follows from Lemma~\ref{result:large_b_small_angle}. By definition of $A$, $(\eta_kA^\top)_i$ represents the number of $\theta_j$ that are smaller than $\theta_i$ minus the number of $\theta_j$ that are larger than $\theta_i$. Therefore, $\eta_kA^\top$ and $\eta_{k'}A^\top$ are independent random permutations of the deterministic vector $[-(m-1), -(m-3), \dots, m-3, m-1]$. Without loss of generality, assume $\eta_kA^\top = [-(m-1), -(m-3), \dots, m-3, m-1]$ and denote $\eta_{k'}A^\top$ by $x$ which is a random permutation of $\eta_kA^\top$. Let $Z = \eta_kA^\top A \eta_{k'}^\top = -(m-1)x_1+\dots+(m-1)x_m$ and define the martingale $y_i = \E{Z|x_1, \dots, x_i}$. In particular, we have $y_0 = \E{Z} = 0$ and $y_m = Z$. We also note that $|y_{i+1}-y_i|\leq 2m^2$. By Azuma's inequality,
  \begin{align*}
    \p{|Z|\geq t} = \p{|y_m-y_0|\geq t}\leq 2e^{-\frac{t^2}{8m^5}},
  \end{align*}
  and thus $|Z|\leq C_3m^{5/2}\log^{1/2}m$ with high probability. Plugging it into \eqref{equ:eta_seq}, we get $||(\eta_k-\eta_{k'})V||_2\geq C_2m$.

  Combining the above two steps, we conclude that
  \begin{align*}
    ||\bar{S}_k-\bar{S}_{k'}|| \geq C(1-\epsilon)m.
  \end{align*} 

\end{document}